\documentclass[12pt, a4paper]{article}

\usepackage[a4paper]{geometry}
\usepackage{pgfplots}
\pgfplotsset{
    x tick style={color=black},
    y tick style={color=black}
}

\usepackage{url}
\usepackage{listings}
\usepackage{amssymb}
\usepackage{graphicx}
\usepackage{algcompatible}
\usepackage{algorithm}
\usepackage{url}
\usepackage{rotating}
\usepackage{booktabs}
\usepackage{subfig}
\usepackage{amsmath}
\usepackage{bbm}

\renewcommand{\labelenumi}{(\alph{enumi})}
\renewcommand\theenumi\labelenumi

\usepackage[english]{babel}
\usepackage{lscape}

\usepackage[utf8]{inputenc}
\usepackage{xspace}
\usepackage{amsmath,amsthm,amssymb,mathtools}
\usepackage{lmodern}

\usepackage[algo2e,ruled,vlined,linesnumbered]{algorithm2e}

\usepackage{xcolor}
\usepackage{tikz}
\usepackage{graphicx}
\usepackage{booktabs} 
\usepackage{xspace}
\usepackage[algo2e,ruled,vlined,linesnumbered]{algorithm2e}
\usepackage{pdfpages}
\usepackage{multirow}
\usepackage{siunitx}
\usepackage{soul}
\usepackage{subfig}
\usepackage{algcompatible}
\usepackage{algorithm}
\usepackage{arydshln} 
\usepackage{ragged2e}

\renewcommand{\labelenumi}{\theenumi}
\renewcommand{\theenumi}{(\roman{enumi})}

\newtheorem{theorem}{Theorem}
\newtheorem{lemma}[theorem]{Lemma}
\newtheorem{corollary}[theorem]{Corollary}
\newtheorem{definition}[theorem]{Definition}

\newcommand{\oea}{\mbox{$(1 + 1)$~EA}\xspace}

\newcommand{\olea}{($1\overset{+}{,}\lambda)$~EA\xspace}

\newcommand{\mpoea}{\mbox{$(\mu+1)$~EA}\xspace}

\newcommand{\opllga}{\mbox{$(1+(\lambda,\lambda))$~GA}\xspace}
\newcommand{\NSGA}{NSGA\nobreakdash-II\xspace}

\newcommand{\onemax}{\textsc{OneMax}\xspace}
\newcommand{\LO}{\textsc{Leading\-Ones}\xspace}
\newcommand{\leadingones}{\LO}

\newcommand{\jump}{\textsc{Jump}\xspace}
\newcommand{\cocz}{\textsc{COCZ}\xspace}

\newcommand{\zplg}{\textsc{ZPLG}\xspace}
\newcommand{\spg}{\textsc{SPG}\xspace}
\newcommand{\oneminmax}{\textsc{OneMinMax}\xspace}
\newcommand{\lotz}{\textsc{LOTZ}\xspace}
\newcommand{\wlptno}{\textsc{WLPTNO}\xspace}
\newcommand{\decobjmop}{\textsc{Dec-obj-MOP}\xspace}

\newcommand{\ojzj}{\textsc{OneJumpZeroJump}\xspace}
\newcommand{\DLB}{\textsc{DeceptiveLeadingBlocks}\xspace}
\newcommand{\mut}{\textsc{mut}\xspace}

\newcommand{\R}{\ensuremath{\mathbb{R}}}

\newcommand{\N}{\ensuremath{\mathbb{N}}} 
\newcommand{\Z}{\ensuremath{\mathbb{Z}}}

\begin{document}
{\sloppy
\title{Theoretical Analyses of Multiobjective Evolutionary Algorithms on Multimodal Objectives\thanks{A preliminary short version was published in the proceeding of AAAI 2021~\cite{DoerrZ21aaai}.}}

\author{Weijie Zheng\\ School of Computer Science and Technology\\
International Research Institute for Artificial Intelligence\\
       Harbin Institute of Technology\\ 
       Shenzhen, China\\ 
        Southern University of Science and Technology\\ Shenzhen, China
\and Benjamin Doerr\thanks{Corresponding author.}\\ Laboratoire d'Informatique (LIX)\\ \'Ecole Polytechnique, CNRS\\ Institut Polytechnique de Paris\\ Palaiseau, France}



\maketitle
\begin{abstract}
Multiobjective evolutionary algorithms are successfully applied in many real-world multiobjective optimization problems. As for many other AI methods, the theoretical understanding of these algorithms is lagging far behind their success in practice. In particular, previous theory work considers mostly easy problems that are composed of unimodal objectives.

As a first step towards a deeper understanding of how evolutionary algorithms solve multimodal multiobjective problems, we propose the $\ojzj$ problem, a bi-objective problem composed of two objectives isomorphic to the classic jump function benchmark. We prove that the simple evolutionary multiobjective optimizer (SEMO) with probability one does not compute the full Pareto front, regardless of the runtime. In contrast, for all problem sizes $n$ and all jump sizes ${k \in [4..\frac n2 - 1]}$, the global SEMO (GSEMO) covers the Pareto front in an expected number of $\Theta((n-2k)n^{k})$ iterations. For $k = o(n)$, we also show the tighter bound $\frac 32 e n^{k+1} \pm o(n^{k+1})$, which might be the first runtime bound for an MOEA that is tight apart from lower-order terms. We also combine the GSEMO with two approaches that showed advantages in single-objective multimodal problems. When using the GSEMO with a heavy-tailed mutation operator, the expected runtime improves by a factor of at least $k^{\Omega(k)}$. When adapting the recent stagnation-detection strategy of Rajabi and Witt~\cite{RajabiW22} to the GSEMO, the expected runtime also improves by a factor of at least $k^{\Omega(k)}$ and surpasses the heavy-tailed GSEMO by a small polynomial factor in $k$. Via an experimental analysis, we show that these asymptotic differences are visible already for small problem sizes: A factor-$5$ speed-up from heavy-tailed mutation and a factor-$10$ speed-up from stagnation detection can be observed already for jump size~$4$ and problem sizes between $10$ and $50$. 
Overall, our results show that the ideas recently developed to aid single-objective evolutionary algorithms to cope with local optima can be effectively employed also in multiobjective optimization.
\end{abstract}

%
%



\maketitle

\section{Introduction}
\label{sec:int}
Real-world problems often contain multiple conflicting objectives. For such problems, a single best solution cannot be determined without additional information. One solution concept for such problems is to compute a set of solutions each of which cannot be improved without worsening in at least one objective (Pareto optima) and then let a decision maker choose one of these solutions. With their  population-based nature, multiobjective evolutionary algorithms (MOEAs) are a natural choice for this approach and have indeed been very successfully applied here~\cite{ZhouQLZSZ11}. 

Similar to the situation for single-objective evolutionary algorithms, the rigorous theoretical understanding of MOEAs falls far behind the success of these algorithms in practical applications. The classic works in this area have defined multiobjective, especially bi-objective, counterparts of well-analyzed single-objective benchmark functions used in evolutionary computation theory and have analyzed the performance of mostly very simple MOEAs on these benchmarks.
 
For example,  {in the problems \cocz~\cite{LaumannsTZWD02} and \oneminmax~\cite{GielL10}}, the two objectives are both (conflicting) variants of the classic \onemax benchmark. The classic benchmark \LO was used to construct the \lotz~\cite{LaumannsTZ04} and \wlptno~\cite{QianYZ13} problems. These multiobjective benchmark problems are among the most intensively studied~\cite{Giel03,BrockhoffFN08,DoerrKV13,DoerrGN16,BianQT18ijcaigeneral,HuangZCH19,HuangZ20,OsunaGNS20,ZhengLD22,BianQ22,ZhengD22gecco}. We note that these problems are all very easy. They are composed of {unimodal} objectives and they have the further property that from each set of solutions $P$ a set $P'$ witnessing the Pareto front can be computed by repeatedly selecting a solution from $P$, flipping a single bit in it, adding it to~$P$, and removing dominated solutions from~$P$. They are thus relatively easy to solve, as witnessed by typical runtimes such as $O(n^2 \log n)$ on the \onemax-type problems or $O(n^3)$ on the \leadingones-inspired benchmarks. These runtimes naturally are higher than for the single-objective counterparts due to the fact the Pareto front of the multiobjective versions has size $\Theta(n)$, hence $\Theta(n)$ times more solutions have to be computed compared to the single-objective setting.
  	
We defer a detailed discussion of the state of the art of rigorous analyses of MOEAs to Section~\ref{sec:previous} and state here only that, to the best of our knowledge, there is not a single runtime analysis work discussing in detail how MOEAs cope with problems composed of multimodal objectives.

\textbf{Our contributions.} This paper aims at a deeper understanding of how MOEAs cope with multiobjective problems consisting of natural, well-analyzed, multimodal objectives. In the theory of single-objective evolutionary computation, the class of \jump functions proposed by Droste et al.~\cite{DrosteJW02} is a natural and intensively used multimodal function class that has inspired many ground-breaking results. 
Hence, in this paper, we design a bi-objective counterpart of the \jump function with problem size~$n$ and jump size~$k$, called $\ojzj_{n,k}$. It consists of one \jump function w.r.t. the number of ones and one \jump function w.r.t. the number of zeros, hence both objectives are isomorphic to classic jump functions. We compute its Pareto front in Theorem~\ref{thm:paretoset} and, as intended, observe that different from the easy multimodal benchmarks the Pareto front is not connected, that is, one cannot transform any solution on the Pareto front into any other solution on the Pareto front via one-bit flips. From this observation, we easily show that for all $n \in \N$ and all $k\in[2..\frac n2]$, the simple evolutionary multiobjective optimizer (SEMO) cannot find the full Pareto front of the \ojzj benchmark (Theorem~\ref{thm:semoojzj}).

The heart of this work are mathematical runtime analyses for the global SEMO (GSEMO). We first show that the classic version of this algorithm for all $n$ and $k \in [2..\lfloor \frac n2 \rfloor]$ finds the Pareto front of $\ojzj_{n,k}$ in $O((n-2k+3)n^{k})$ iterations (and fitness evaluations) in expectation (Theorem~\ref{thm:gsemoojzj}). We show a matching lower bound of $\Omega((n-2k)n^{k})$ for $k\in [4..\frac n2-1]$ (Theorem~\ref{thm:lowerbound}). Here and in the remainder, the asymptotic notation only hides constants independent of $n$ and $k$. We note that our actual bounds are more precise. In particular, for $4 \le k = o(n)$, the expected runtime is $\frac 32 e n^{k+1}$ apart from lower-order terms. This result might be the first runtime analysis of an MOEA that is tight apart from lower order terms.

We then try to reuse two ideas that led to performance gains for multimodal problems in single-objective optimization. Doerr et al.~\cite{DoerrLMN17} proposed the \emph{fast mutation} operator in which the number of flipped bits follows a power-law distribution with (negative) exponent $\beta > 1$. They showed that the \oea with this operator optimizes jump functions with jump size~$k$ by a factor of $k^{\Omega(k)}$ faster than the standard \oea. We show that the GSEMO with fast mutation also witnesses such an improvement over the standard GSEMO. More precisely, we prove an upper bound of $O((n-2k)(en)^{k}/k^{k+0.5-\beta})$ for this algorithm (Theorem~\ref{thm:gsemohojzj}). We also provide a result with explicit leading constant, for the first time for an algorithm using fast mutation. We are optimistic that our non-asymptotic estimates for the number of flipped bits by this operator will be useful also in other works.

The stagnation-detection strategy of Rajabi and Witt~\cite{RajabiW22} is a second way to cope with multimodality (in single-objective optimization). So far only used in conjunction with the \olea (and there mostly with the \oea), it consists of counting for how long no improvement has been found and using this information to set the mutation rate (higher mutation rates when no improvement was made for longer time). With suitable choices of the parameters of this strategy, the \oea can optimize jump functions with jump size~$k$ faster than the standard \oea by a factor of $\Omega(k^{\Omega(k)})$. This is the same rough estimate as for the \oea with fast mutation. A more detailed calculation, however, shows that the stagnation-detection \oea is faster than the fast \oea by a factor of $\Theta(k^{\beta-0.5})$. We note that this difference usually is not extremely large since $\beta$ is usually chosen small ($\beta = 1.5$ was recommended by~\cite{DoerrLMN17}) and $k$ has to be small to admit reasonable runtimes (recall that the runtime dependence on $n$ is $\Omega(n^k)$). Nevertheless, in experiments conducted by \cite{RajabiW22} the stagnation-detection \oea was faster than the fast \oea by around a factor of three. Different from fast mutation, it is not immediately clear how to incorporate stagnation detection into the GSEMO. Being an algorithm with non-trivial parent population, one question is whether one should count unsuccessful iterations globally or individually for each parent individual. Also, clearly the parameters of the algorithm need to be adjusted to the longer waiting times for an improvement. We succeeded in defining a stagnation-detection version of the GSEMO that effectively solves the \ojzj problem, more precisely, that computes the full Pareto front in an expected runtime of $O((n-2k)(en)^{k}/k^{k})$, again a $k^{\Omega(k)}$ factor improvement over the classic GSEMO and reducing the runtime guarantee for the heavy-tailed GSEMO by a small factor of $\Omega(k^{\beta-0.5})$, see Theorem~\ref{thm:sdgsemoojzj}.

Via a small set of experiments, we demonstrate that these are not only asymptotic differences. Compared to the standard GSEMO, we observe roughly a factor-$5$ speed-up with heavy-tailed mutation and a factor-$10$ speed-up with our stagnation-detection GSEMO, and this already for jump size~$k=4$ and problem sizes $n$ between $10$ and $50$. 

 {We note that this work is an extended version of the 7-page (excluding acknowledgments and references) preliminary paper~\cite{DoerrZ21aaai} and differs in the following ways. We estimate the leading constant of the asymptotic upper bound of the runtime of the GSEMO with heavy-tailed mutation for not too large jump size, which has not been shown in our preliminary version, and also not been shown before for the theoretical works with the heavy-tailed mutation. Besides, this version adds a formal discussion on the multimodality in multiobjective problems in Section~\ref{sec:previous}, and contains all mathematical proofs that had to be omitted in~\cite{DoerrZ21aaai} for reasons of space.}

The remainder of this paper is organized as follows. Section~\ref{sec:basic} introduces the basic definitions that will be used throughout this work. In Section~\ref{sec:previous}, we discuss the relevant previous works. The $\ojzj_{n,k}$ function class is defined in Section~\ref{sec:ojzj}. Our mathematical runtime analyses are presented in Sections~\ref{sec:semo} to~\ref{sec:sdgsemo}. Section~\ref{sec:exp} shows our experimental results. A conclusion is given in Section~\ref{sec:con}.

\section{Basic Definitions}
\label{sec:basic}

A multiobjective optimization problem consists of optimizing multiple objectives simultaneously. In this paper, we restrict ourselves to the maximization of bi-objective pseudo-Boolean problems $f=(f_1,f_2): \{0,1\}^n \rightarrow \R^2$. We note that only sporadic theoretical works exist that discuss the optimization of more than two objectives~\cite{LaumannsTZ04,BianQT18ijcaigeneral,HuangZLL21}.

We brief{}ly review the standard notation in multiobjective optimization. For any two search points $x, y \in \{0,1\}^n$, we say that
\begin{itemize}
\item $x$ \emph{weakly dominates} $y$, denoted by $x \succeq y$, if and only if $f_1(x) \ge f_1(y)$ and $f_2(x) \ge f_2(y)$;
\item $x$ \emph{dominates} $y$, denoted by $x \succ y$, if and only if $f_1(x) \ge f_1(y)$ and $f_2(x) \ge f_2(y)$ and at least one of the inequalities is strict.
\end{itemize}
We call $x\in\{0,1\}^n$ \emph{Pareto optimal} if and only if there is no $y\in\{0,1\}^n$ such that $y\succ x$. All Pareto optimal solutions form the \emph{Pareto set} $S^*$ of the problem. The set $F^* = f(S^*) = \{(f_1(x),f_2(x)) \mid x \in S^*\}$ of the function values of the Pareto set is called the \emph{Pareto front}. 

For most multiobjective problems, the objectives are at least partially conflicting and thus there is usually not a single Pareto optimum. Since a priori it is not clear which of several incomparable Pareto optima to prefer, the most common target is to compute the Pareto front, that is, compute a set $P$ of solutions such that $f(P) := \{(f_1(x),f_2(x)) \mid x \in P\}$ equals the Pareto front~$F^*$. This is our objective in this work as well. We note that if the Pareto front is excessively large, then one has to resort to approximating it in a suitable manner, but this will not be our problem here. 

We will use $|x|_1$ and $|x|_0$ to denote the number of ones and the number of zeros of the search point $x \in \{0,1\}^n$. We use $[a..b]$ to denote the set $\{a,a+1,\dots,b\}$ for $a,b\in \Z$ and $a\le b$.

\section{Previous Works and A Discussion of Multimodality in Multiobjective Problems}\label{sec:previous}

Since this work conducts a mathematical runtime analysis of an evolutionary algorithm for discrete search spaces, we primarily discuss the state of the art inside this subarea of evolutionary computation. 
The mathematical runtime analysis as a rigorous means to understand evolutionary algorithms was started in the 1990s with works like~\cite{Muhlenbein92,Back93,Rudolph97}. The first analysis of MOEAs were conducted by Laumanns et al.~\cite{LaumannsTZWD02,LaumannsTZ04}, Giel~\cite{Giel03}, and Thierens~\cite{Thierens03}. The theory of MOEAs here often followed the successful examples given by single-objective works. For example, the multiobjective benchmarks \cocz \cite{LaumannsTZ04}, and \oneminmax~\cite{GielL10} are composed of two contradicting copies of the classic \onemax problem, the problems \lotz~\cite{LaumannsTZ04} and \wlptno~\cite{QianYZ13} follow the main idea of the \leadingones problem~\cite{Rudolph97}.

Due to the higher complexity of analyzing MOEAs, many topics with interesting results on single-objective EAs received almost no attention for MOEAs. For example, the first runtime analysis of a single-objective algorithm using crossover was conducted by Jansen and Wegener~\cite{JansenW01} and was followed up by a long sequence of important works. In contrast, only sporadic works discuss crossover in MOEAs~\cite{NeumannT10,QianYZ13,HuangZCH19,BianQ22}. Dynamic parameter settings for single-objective EAs were first discussed by Jansen and Wegener~\cite{JansenW05} and then received a huge amount of attention (see, e.g., the bookchapter~\cite{DoerrD20bookchapter}), whereas the only such result for MOEAs appears to be one result {by Doerr et al.~\cite{DoerrHP22}}. How single-objective evolutionary algorithms cope with noise was discussed already by Droste~\cite{Droste04}, but despite numerous follow-up works in single-objective optimization, there is only a single mathematical runtime analysis of an MOEA in the presence of noise {~\cite{Gutjahr12}}.

In this work, we regard another topic that is fairly well-understood in the evolutionary theory community on the single-objective side, but where essentially nothing is known on the multiobjective side. It is well-known in general that local optima are challenging for evolutionary algorithms.
Already one of the first runtime analysis works, the ground-breaking paper by Droste et al.~\cite{DrosteJW02}, proposes a multimodal (that is, having local optima different from the global optimum) benchmark of scalable difficulty. The jump function with difficulty parameter~$k$, later called \emph{gap size}, resembles the easy \onemax function except that it has a valley of very low fitness around the optimum. This valley consists of the $k-1$ first Hamming levels around the optimum, hence the search points in distance $k$ form a local optimum that is easy to reach, but hard to leave. Mutation-based algorithms need to flip the right $k$ bits to go from any point on the local optimum to a search point of better fitness (which here is the global optimum). While other multimodal benchmarks have been proposed and have received some attention, most notably trap functions~\cite{JagerskupperS07} and the \DLB problem~\cite{LehreN19foga}, it is safe to say that the \jump benchmark is the by far most studied one and that these studies have led to fundamental insights on how evolutionary algorithms cope with local optima.

For reasons of space, we only discuss some of the most significant works on jump functions. The first work~\cite{DrosteJW02} determined the runtime of the \oea on \jump functions with gap size~$k \ge 2$ to be $\Theta(n^k)$, showing that local optima can lead to substantial performance losses. This bound was tightened and extended to arbitrary mutation rates by Doerr et al.~\cite{DoerrLMN17}. This analysis led to the discovery of a heavy-tailed mutation operator and more generally, established the use of power-law distributed parameters, an idea that led to many interesting results subsequently, e.g., \cite{AntipovBD21gecco,CorusOY21foga,AntipovBD22,DangELQ22}. Again an analysis of how single-trajectory heuristics optimize jump functions led to the definition of a stagnation-detection mechanism, which in a very natural manner leads to a reduction of the time an EA is stuck on a local optimum~\cite{RajabiW22}. Jump functions are among the first and most convincing examples that show that crossover can lead to significant speed-ups over only using mutation as variation operator~\cite{JansenW02,DangFKKLOSS18,AntipovDK22}. That estimation-of-distribution algorithms and ant-colony optimizers may find it easier to cope with local optima was also shown via jump functions~\cite{HasenohrlS18,Doerr21cgajump,BenbakiBD21,Witt23}.

In contrast to this intensive discussion of multimodal functions in single-objective evolutionary optimization, almost no such results exist in multiobjective optimization. Before discussing this aspect in detail, let us agree (for this work) on the following terminology. We say that a function $f : \{0,1\}^n \to \R$ is \emph{unimodal} if it has a unique maximum and if all other search points have a Hamming neighbor with strictly larger fitness. We say that $f$ is \emph{multimodal} if it has local optima different from the global optimum. Here we call \emph{local optimum} a set $S$ of search points such that $f(x) = f(y)$ for all $x, y \in S$ and such that all neighbors of $S$, that is, all $z \notin S$ having a Hamming neighbor in $S$, have a smaller $f$-value.   {Note that a global optimum is also a local optimum according to this definition}. With these definitions, unimodal functions are those where a hillclimber can reach the optimum via one-bit flips regardless of the initial solution. In contrast, multimodal functions definitely need a single-trajectory heuristic to flip more than one bit (when started in a suitable solution) or need the acceptance of truly inferior solutions. With our definition, plateaus of constant fitness render a function not unimodal, but they do not imply multimodality. This is convenient for our purposes as we shall not deal with such plateaus -- for the simple reason that they behave again differently from unimodal functions and true local optima. We note that the role of plateaus in multiobjective evolutionary optimization has been extensively studied, among others, in~\cite{BrockhoffFHKNZ07,FriedrichHN10,FriedrichHN11,QianTZ16,LiZZZ16}.

We extend the definitions of unimodality and multimodality to multiobjective problems in the natural way: A multiobjective problem is \emph{unimodal} if all its objectives are unimodal functions. Many such problems are easy and even simple MOEAs employing one-bit flips as only variation operator, e.g., the SEMO, can find the full Pareto front. Examples of such problems include the well-studied benchmarks inspired by the unimodal single-objective benchmarks \onemax and \leadingones. However, somewhat surprisingly, it is not true that the SEMO can cover the full Pareto front of all unimodal multiobjective problems as we now show. 

\begin{lemma}
There exists a unimodal multiobjective problem whose full Pareto front cannot be covered by the SEMO in an arbitrarily long time with a positive probability.
\label{lem:cesemo}
\end{lemma}

\begin{proof}
We consider the following bi-objective problem defined on $\{0,1\}^3$. 
\begin{align*}
&f(000)=(0,1), f(001)=(2,0), f(010)=(0,2), f(011)=(1,5)\\
&f(100)=(4,0), f(101)=(5,0), f(110)=(3,3), f(111)=(0,4).
\end{align*} 
Then its Pareto front is $\{(1,5),(5,0),(3,3)\}$.
We know that $f_1$ is unimodal as 
\begin{align*}
&f_1(000)<f_1(001), f_1(001)<f_1(101),
f_1(010)<f_1(011), f_1(011)<f_1(001), \\
&f_1(100)<f_1(101), f_1(110)<f_1(100), f_1(111)<f_1(110),
\end{align*} 
and $f_1(101)$ is the unique maximal value. We also know that $f_2$ is unimodal as 
\begin{align*}
& f_2(000)<f_2(010), f_2(001)<f_2(000), f_2(010)<f_2(011), f_2(100)<f_2(000),\\
& f_2(101)<f_2(111), f_2(110)<f_3(111), f_2(111)<f_2(011),
\end{align*} 
and $f_2(011)$ is the unique maximal value.

Consider the run of the SEMO described in Table~\ref{tbl:cesemo}. It is easy to see that this process happens with probability of
\begin{align*}
\frac18 \cdot \frac 13 \cdot \left(\frac 12 \frac13\right) \cdot \left(\frac12 \frac13 \right)\cdot \left(\frac12 \frac13 \right) = \frac1{5184}>0.
\end{align*}
\begin{table}[!h]
\centering
\setlength{\tabcolsep}{7mm}
\begin{tabular}{c l }
\toprule
Generation & Population\\
\midrule
0 & (100, (4,0)) \\
1 & (100, (4,0)),\, (000, (0,1)) \\
2 & (100, (4,0)),\, (010, (0,2)) \\
3 & (101, (5,0)),\, (010, (0,2)) \\
4 & (101, (5,0)),\, (011, (1,5)) \\
\bottomrule
\end{tabular}
\caption{An example process of the SEMO. Given is the population at the end of each iteration (which for iteration~$0$ is the initial population). To ease understanding the process, with each individual we also state its objective value.}
\label{tbl:cesemo}
\end{table}

We now show that once the optimization process has reached the state described in the last line of Table~\ref{tbl:cesemo}, the Pareto front point $(3,3)$ cannot be reached anymore. Note that $110$ is the only Pareto optimum for $(3,3)$, and its Hamming neighbors are $010,100$, and $111$. Since $f(010)=(0,2),f(100)=(4,0)$, and $f(111)=(0,4)$, we know that $010$ and $111$ are strictly dominated by $011$, and that $100$ is strictly dominated by $101$. Hence, even if any of them is generated, it cannot survive to the next generation when the current population is $\{101,011\}$. Therefore, $110$ cannot be generated (and $(3,3)$ cannot be reached) via the SEMO with the one-bit mutation.
\end{proof}

We say that a multiobjective problem is \emph{multimodal} if at least one objective is multimodal. We note that this does not automatically imply that the SEMO cannot find the full Pareto front, in fact, as the following proof of Lemma~\ref{lem:mmeasy} shows, there are multi-objective problems consisting only of multimodal objectives such that the SEMO regardless of the initial solution computes the full Pareto front very efficiently. Clearly, our interest in this work are multimodal multiobjective problems that are harder to optimize. To avoid misunderstandings, we note that a second notion of multimodality exists. In \emph{multimodal optimization}, the target is to find all or a diverse set of solutions having some optimality property~\cite{Preuss15,LiangYQ16,TanabeI20}. This notion is unrelated to our notion of multimodality. 

\begin{lemma}\label{lem:mmeasy}
  There is a bi-objective problem $f = (f_1,f_2) : \{0,1\}^n \to \R^2$ such that both objectives $f_1$ and $f_2$ have several local optima, but the SEMO algorithm, regardless of its initial solution, finds the full Pareto front in time $O(n^2 \log n)$.
\end{lemma}

\begin{proof}
  Let $k, \ell \in \N_{\ge 2}$ and $n = k\ell$. Define $f_1$ via
  \begin{align*}
  f_1(x)&=
  \begin{cases}
  |x|_1, & \text{if~} |x|_1 \equiv 0 \pmod \ell,\\ 
  \lfloor |x|_1 / \ell \rfloor \ell + (\ell-i), & \text{if~} |x|_1 \equiv i \pmod \ell,
  \end{cases}
  \end{align*}
	for all $x \in \{0,1\}^n$. Hence $f_1$ agrees with the classic \onemax benchmark when $|x|_1$ is a multiple of $\ell$. Solutions $x$ with $|x|_1 \equiv 1 \pmod \ell$ are local optima different from the global optimum, which is $x^* = 1^n$. The function $f_1$ is deceptive in each interval $[j\ell+1..j\ell+\ell-1]$, $j = 0, \dots, k-1$.
	
	We define $f_2$ analogously, but with the roles of zeroes and ones interchanged, that is,
  \begin{align*}
  f_2(x)&=
  \begin{cases}
  |x|_0, & \text{if~} |x|_0 \equiv 0 \pmod \ell,\\ 
  \lfloor |x|_0 / \ell \rfloor \ell + (\ell-i), & \text{if~} |x|_0 \equiv i \pmod \ell,
  \end{cases}
  \end{align*}
	for all $x \in \{0,1\}^n$. Clearly, $f_2$ has the same characteristics as $f_1$, in particular, it is highly multimodal with its $k$ local optima. 

Figure~\ref{fig:lemma2} plots these two functions with $(n,\ell)=(50,10)$. We clearly to see the $n/\ell+1=6$  {local optima} for both objectives, namely the points with $|x|_1 \equiv 1 \pmod \ell$ and $x=1^n$ for $f_1$ and the points with $|x|_0 \equiv 1 \pmod \ell$ and $x=0^n$ for $f_2$.
\begin{figure}[!ht]
\centering
\includegraphics[width=3.8in]{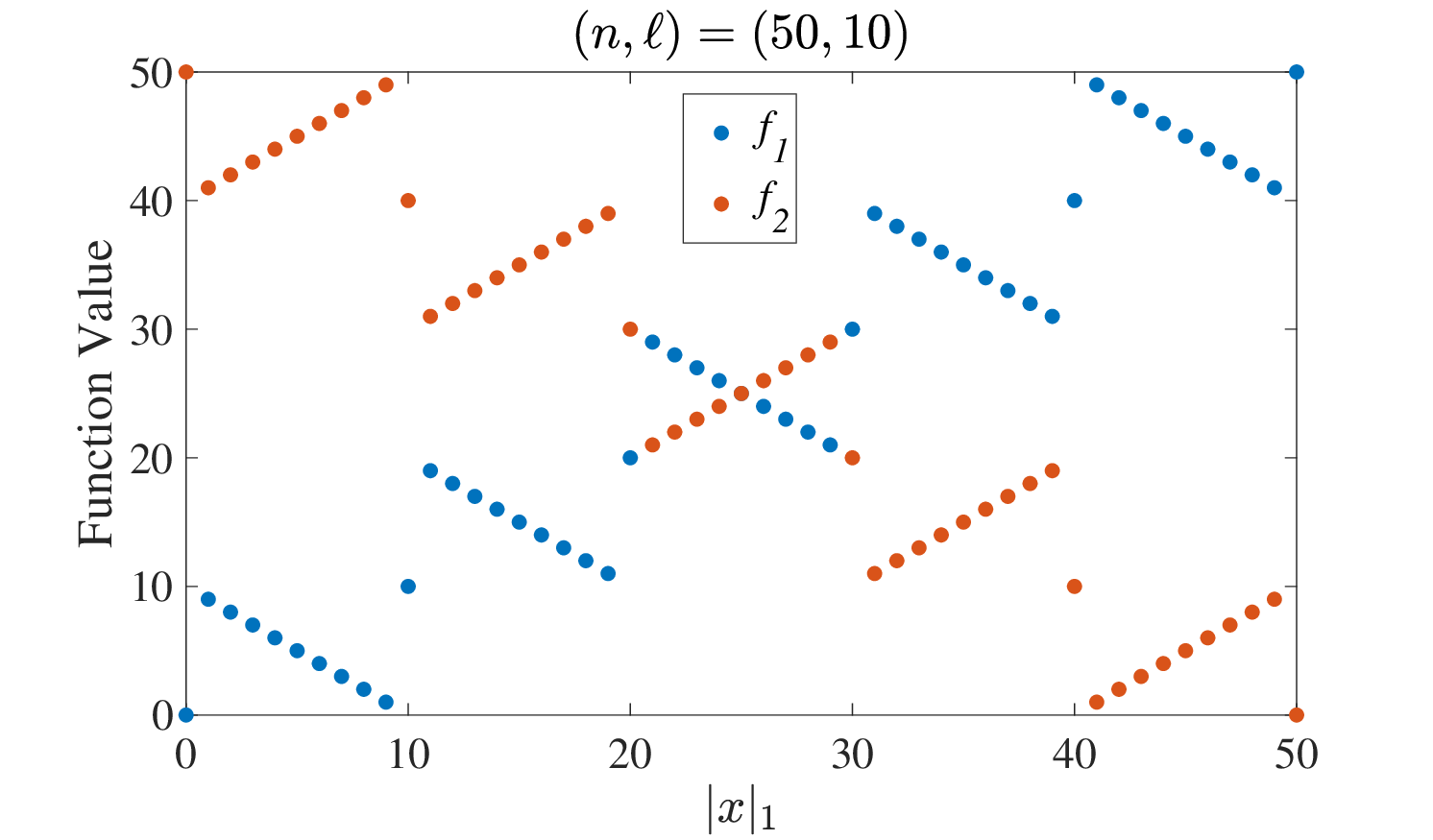}
\caption{The values of $f_1$ and $f_2$ with respect to $|x|_1$, the number of ones in the search point $x$.}
\label{fig:lemma2}
\end{figure}

  We note that any solution is a Pareto optimum of $f = (f_1,f_2)$. Indeed, we have $f_1(x) + f_2(x) = n$ for all $x \in \{0,1\}^n$. Consequently, if $f_1(x) > f_1(y)$ for some $x, y \in \{0,1\}^n$, then necessarily $f_1(x) < f_2(y)$ and thus $x$ does not dominate $y$. 
	
	Since all solutions are Pareto optima and since $f$ has the same classes of search points of equal fitness, the optimization process of the SEMO on $f$ is identical to the one on \onemax. With~\cite[Theorem~3]{GielL10} and its proof, our claim is proven.
\end{proof}

Having made precise what we understand as multimodal multiobjective problem, we review the known results on such problems.

To the best of our knowledge, only very few runtime analysis works on multimodal multiobjective problems exist and they all are not focussed on the possible difficulties arising from the multimodality. Roughly speaking, these can be divided into works that construct artificial problems to demonstrate, often in an extreme way, a particular strength or weakness of an algorithm, and works on classic combinatorial optimization problems that happen to be multimodal. The first set of works gave rise to the multimodal problems \zplg, \spg, and \decobjmop. The first two of these were proposed by Qian et al.~\cite{QianTZ16} and are defined as follows.
\begin{definition}[\zplg~\cite{QianTZ16}]
\label{def:zplg}
The function $\zplg: \{0,1\}^n \rightarrow \R\times\{0,1,2\}$ is defined by
\begin{align*}
\zplg(x)&=
\begin{cases}
(n+1,1), & \text{if}~x=1^i0^{n-i}, i \in [1..\frac34 n-1];\\
(n+2+i,0), & \text{if}~x=1^{\frac34n +2i}0^{\frac14n-2i}, i \in [0..\frac18 n];\\
(|x|_0,2), & \text{else}.
\end{cases}
\end{align*}
\end{definition}

\begin{definition}[\spg~\cite{QianTZ16}]
\label{def:spg}
The function $\spg: \{0,1\}^n \rightarrow \R\times\{0,1\}$ is defined by
\begin{align*}
\spg(x)&=
\begin{cases}
(-1,0), & \text{if}~x=1^i0^{n-i}, i~\textup{mod}~3 =1, i\in[1..n];\\
(in,0), & \text{if}~x=1^i0^{n-i}, i~\textup{mod}~3 \in\{0,2\}, i\in[1..n];\\
(|x|_0,1), & \text{else}.
\end{cases}
\end{align*}
\end{definition}
In order to investigate the effect of mixing low-level heuristics, Qian et al.~\cite{QianTZ16} showed that mixing fair selection w.r.t.\ the decision space and the objective space is beneficial for \zplg, and that mixing 1-bit and 2-bit mutation is efficient for \spg. 

In the first theoretical study of decomposition-based MOEAs, Li et al.~\cite{LiZZZ16} analyzed how the MOEA/D solves several multiobjective problems, among them the following multimodal one. 
\begin{definition}[\decobjmop~\cite{LiZZZ16}]
\label{def:dec}
The function $\decobjmop: \{0,1\}^n \rightarrow \R^2$ is defined by
\begin{align*}
\decobjmop(x)=(n+1-|x|_0~\textup{mod}~n+1, n+|x|_0~\textup{mod}~n+1).
\end{align*}
\end{definition}

We note that the three problems \zplg, \spg, and \decobjmop all appear slightly artificial. Also, they have only one multimodal objective. The first objective of \zplg has  {local optima} on $x=1^{\frac34n +2i}0^{\frac14n-2i}, i \in [0..\frac18 n]$,  the first objective of \spg has  {local optima} on $x=0^n$ and $x=1^i0^{n-i}, i~\textup{mod}~3=0, i\in[1..n]$, and the second objective of \decobjmop has two  {local optima} on $x \in \{0^n,1^n\}$.

The second type of works dealing with multimodal multiobjective problems are those which regard combinatorial optimization problems, for example,~\cite{LaumannsTZ04knapsack,KumarB06,Neumann07,NeumannR08,Greiner09,Horoba09,QianYZ13,QianYZ15ijcai,QianZTY18,FengQT19,QianBF20,RoostapourBN20,RoostapourNNF22}. Combinatorial optimization problems almost always are multimodal, simply because already a simple cardinality constraint suffices to render a problem multimodal. We note for some such problems the multimodality is not very strong, e.g., for minimum spanning tree problems flipping two bits suffices to leave a local optimum. Overall, all these works are relatively problem-specific and we could not distill from these any general insights on how MOEAs cope with local optima.

\section{The $\ojzj$ Problem}
\label{sec:ojzj}

To study via mathematical means how MOEAs cope with multimodality, we now define and analyze a class of bi-objective functions of scalable difficulty.  {As mentioned above, t}his is strongly influenced by the single-objective \jump function class proposed in~\cite{DrosteJW02}, which is intensively used in the theory of single-objective evolutionary computation and which gave rise to many interesting results including that larger mutation rates help in the optimization of multimodal functions, like~\cite{DoerrLMN17}, that crossover can help to cope with multimodality, like~\cite{JansenW02,DangFKKLOSS18}, and that estimation-of-distribution algorithms and the \opllga can significantly outperform classic evolutionary algorithms on multimodal problems, like~\cite{HasenohrlS18,Doerr21cgajump,AntipovD20ppsn,AntipovDK22}.

We recall that for all $n \in \N$ and $k \in [1..n]$, the jump function $\jump_{n,k} : \{0,1\}^n \to \R$ with problem size~$n$ and gap size~$k$ is defined by $\jump_{n,k}(x) = k+|x|_1$, if $|x|_1 \in [0..n-k] \cup \{n\}$ and $\jump_{n,k}(x) = n - |x|_1$ otherwise. Hence for $k \ge 2$, this function has a valley of low fitness around its optimum~$x^*=1^n$, which can be crossed only by flipping $k$ specific bits (or accepting solutions with very low fitness). We define the $\ojzj_{n,k}$ function as a bi-objective counterpart of the function $\jump_{n,k}$. 
\begin{definition}[$\ojzj_{n,k}$]
Let $n \in \N$ and $k=[1..n]$. The function $\ojzj_{n,k}=(f_1,f_2): \{0,1\}^n \rightarrow \R^2$ is defined by
\begin{align*}
f_1(x)&=
\begin{cases}
k+|x|_1, & \text{if~} |x|_1\le n-k \text{~or~} x=1^n,\\
n-|x|_1, & \text{else};
\end{cases}
\\
f_2(x)&=
\begin{cases}
k+|x|_0, & \text{if~} |x|_0\le n-k \text{~or~} x=0^n,\\
n-|x|_0, & \text{else}.
\end{cases}
\end{align*}
\end{definition}
Hence the first objective of $\ojzj_{n,k}$ is just the classic $\jump_{n,k}$ function. The second objective has a fitness landscape isomorphic to this function, but the roles of zeros and ones are interchanged, that it, $f_2(x) = \jump_{n,k}( {\bar{x}})$ {, where $\bar{x}$ is the bitwise complement of $x$}. Figure~\ref{fig:ojzj} displays these two functions and in particular the two  {local optima} on $|x|_1=n-k$ and $x=1^n$ for the first objective and two  {local optima} on $|x|_1=k$ and $x=0^n$ for the second objective.
\begin{figure}[!ht]
\centering
\includegraphics[width=3.8in]{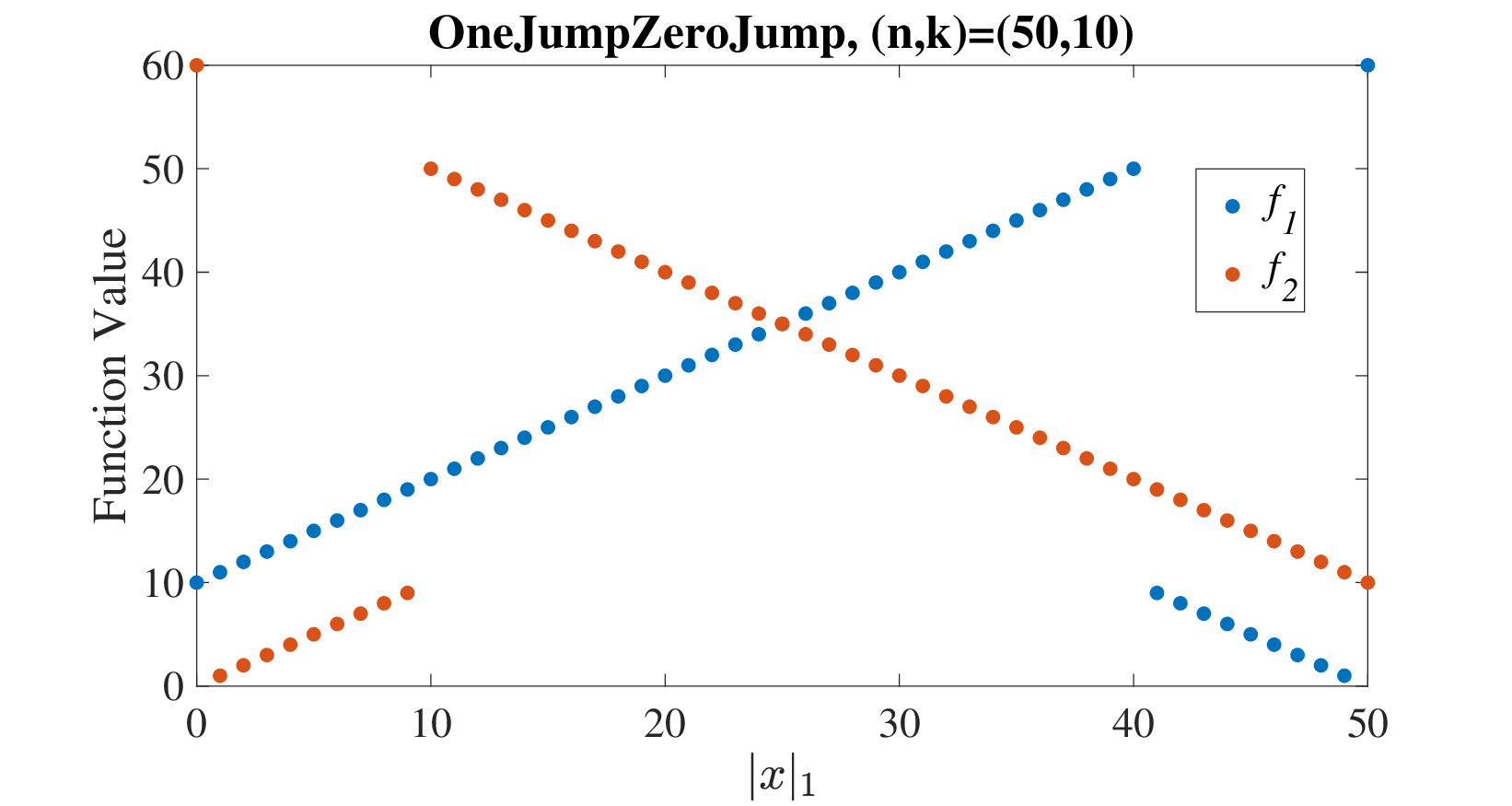}
\caption{The values of two objectives in $\ojzj_{n,k}$ with respect to $|x|_1$, the number of ones in the search point $x$.}
\label{fig:ojzj}
\end{figure}

In the following theorem we determine the Pareto set and front of the $\ojzj_{n,k}$ function. As Figure~\ref{fig:ojzj} suggests, for $k \ge 2$ the Pareto set consists of an inner region of all search points $x$ with $|x|_1 \in [k..n-k]$ and the two extremal points $1^n$ and $0^n$, as visualized in Figure~\ref{fig:ojzjPF}.  {We shall call the Pareto front values $\{(a,2k+n-a) \mid a\in [2k..n]\}$ the \emph{inner part/region} of the Pareto front in this paper.}

\begin{theorem}
The Pareto set of the $\ojzj_{n,k}$ function is $S^*=\{x\mid |x|_1 \in [k..n-k] \cup \{0,n\} \}$, and the Pareto front is $F^*=\{(a,2k+n-a) \mid a\in[2k..n] \cup \{k,n+k\}\}$.
\label{thm:paretoset}
\end{theorem}

\begin{proof}
We firstly prove that for any given $x\in S^*$, $y \nsucc x$ holds for all $y \in \{0,1\}^n$. If $y \succ x$, then $f_1(y)\ge f_1(x)$ and $f_2(y) \ge f_2(x)$, and at least one of two inequalities is strict. W.l.o.g., let $f_1(y) > f_1(x)$. From $x\in S^*$, we have $f_1(x) \ge k$, hence $y\ne 0^n$. Since $f_1(y) \le n+k$, we know $x\ne 1^n$. Then we have $|x|_1 \in [k..n-k] \cup \{0\}$ since $x\in S^*$. 
\begin{itemize}
\item If $|x|_1 \in [k..n-k]$, then $k \le |x|_1 < |y|_1 \le n-k$ or $|y|_1 = n$. For $k \le |y|_1 \le n-k$, we have $k \le |y|_0 < |x|_0 \le n-k$. Then $f_2(y) = k+|y|_0 < k+|x|_0 =f_2(x)$. For $|y|_1=n$, we have $f_2(y)=k < k+k \le f_2(x)$ since $|x|_0\in[k..n-k]$ from $|x|_1\in[k..n-k]$.
\item If $|x|_1=0$, then $f_2(x) = k+n > f_2(y)$ since $y \neq 0^n$.
\end{itemize}
That is, we obtain $f_2(y) < f_2(x)$ which is contradictory to $y \succ x$. Hence,  {$S^*$ is the Pareto set.}

We secondly prove that for any given $y \in \{0,1\}^n \setminus S^*$, there is $x \in S^*$ such that $x \succ y$. Since $y \in \{0,1\}^n \setminus S^*$, we have $|y|_1\in[1..k-1] \cup [n-k+1..n-1]$. 
\begin{itemize}
\item If $|y|_1\in[1..k-1] $, we have $f_1(y) = k+|y|_1$ and $f_2(y)=|y|_1$. Then we could take $x$ with $|x|_1=k$, and have $f_1(x) = k+k > f_1(y)$ and $f_2(x)=k+n-k > f_2(y)$. Then $x\succ y$.
\item If $|y|_1\in [n-k+1..n-1]$, we have $f_1(y)=n-|y|_1$ and $f_2(y)=k+|y|_0 \le 2k-1$. Then we could take $x$ with $|x|_1=n-k$, and have $f_1(x)=n>f_1(y)$ and $f_2(x)=k+k>f_2(y)$. Then $x \succ y$.
\end{itemize}
Hence,  {$F^*$ is the Pareto front and this theorem is} proven.
\end{proof}

Theorem~\ref{thm:paretoset} implies that for $k > n/2$, the Pareto front consists only of the two extremal solutions $0^n$ and $1^n$. This appears to be a relatively special situation that is not very representative for multimodal multiobjective problems. For this reason, in the remainder, we shall only regard the case that $k \le n/2$. In this case, again from Theorem~\ref{thm:paretoset}, we easily obtain in the following corollary a general upper bound on the size of any set of solutions without pair-wise weak domination, and thus also on the size of the population in the algorithms discussed in this work.
\begin{corollary}
Let $k\le n/2$. Consider any set of solutions $P$ such that $x\not\preceq y$ w.r.t. $\ojzj_{n,k}$ for all $x,y \in P$ with $x\ne y$. Then $|P|\le n-2k+3$.
\label{cor:population}
\end{corollary}

\begin{proof}
  We first note that any solution that lies on the Pareto front dominates any solution not on the Pareto front. Hence either $P$ is a subset of the Pareto front or it contains no point on the Pareto front. In the first case, Theorem~\ref{thm:paretoset} immediately shows our claim. Hence assume that $P$ contains no point of the Pareto front. Let $x,y \in P$ with $|x|_1, |y|_1 \in [1..k-1]$. Assume without loss of generality that $|x|_1 \le |y|_1$. Then $x \preceq y$. Hence $P$ contains at most one solution $x$ with $|x|_1 \in [1..k-1]$. A symmetric argument shows that $P$ contains at most one solution $x$ with $|x|_1 \in [n-k+1..n-1]$. Hence $|P| \le 2  {~\le n-2k+3}$.
\end{proof}

\begin{figure}[!ht]
\centering
\includegraphics[width=3.8in]{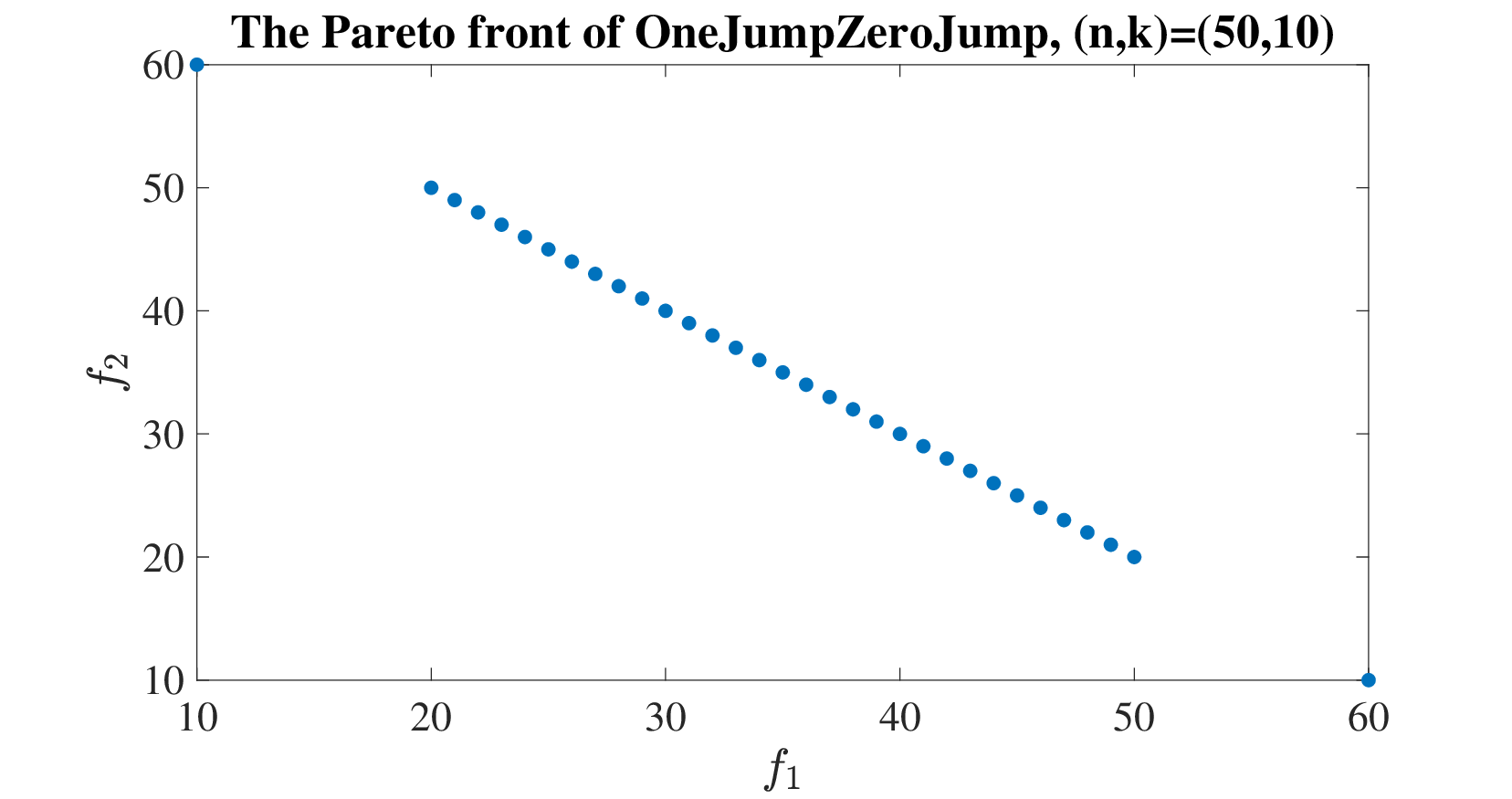}
\caption{The Pareto front for the $\ojzj_{n,k}$ function with $(n,k)=(50,10)$.}
\label{fig:ojzjPF}
\end{figure}

\section{SEMO Cannot Optimize $\ojzj$ Functions}
\label{sec:semo}

The simple evolutionary multiobjective optimizer (SEMO), proposed by~\cite{LaumannsTZWD02}, is a well-studied basic benchmark algorithm in multiobjective evolutionary theory~\cite{QianYZ13,LiZZZ16}. It is a multiobjective analogue of the randomized local search (RLS) algorithm, which starts with a random individual and tries to improve it by repeatedly flipping a single random bit and accepting the better of parent and this offspring. As a multiobjective algorithm trying to compute the full Pareto front, the SEMO naturally has to work with a non-trivial population. This is initialized with a single random individual. In each iteration, a random parent is chosen from the population. It generates an offspring by flipping a random bit. The offspring enters the population if it is not weakly dominated by some individual already in the population. In this case, any individual dominated by it is removed from the population. The details of SEMO are shown in Algorithm~\ref{alg:semo}. We note that more recent works such as~\cite{FriedrichHN10} define the (G)SEMO minimally different, namely so that in case of equal objective values the offspring enters the population and replaces the existing individual with this objective value. Preferring offspring in case of equal objective values allows EAs to traverse plateaus of equal fitness and this is generally preferred over keeping the first solution with a certain objective value. For our purposes, this difference is not important since all points with equal fitness behave identically. 
\begin{algorithm}[!ht]
    \caption{SEMO}
    {\small
    \begin{algorithmic}[1]
    \STATE {Generate $x\in\{0,1\}^n$ uniformly at random and $P\leftarrow \{x\}$}
    \LOOP
    \STATE {Uniformly at random select one individual $x$ from $P$}
    \STATE {Generate $x'$ via flipping one bit chosen uniformly at random}
    \IF {there is no $y \in P$ such that $x' \preceq y$}
    \STATE {$P=\{z\in P \mid z \npreceq x'\} \cup \{x'\}$} 
    \ENDIF
    \ENDLOOP
    \end{algorithmic}
    \label{alg:semo}
    }
\end{algorithm}

In the following Theorem~\ref{thm:semoojzj}, we show that the SEMO cannot cope with the multimodality of the $\ojzj_{n,k}$ function. Even with infinite time, it does not find the full Pareto front of any $\ojzj$ function. We note that this result is not very surprising and we prove it rather for reasons of completeness. It is well-known in single-objective optimization that the \emph{randomized local search} heuristic with positive (and often very high) probability fails on multimodal problems (the only published reference we found is~\cite[Theorem~9]{DoerrJK08}, but surely this was known before). It is for the same reason that the SEMO algorithm cannot optimize the $\ojzj$ function. 

\begin{theorem}
For all $n,k\in \N$ with $k \in [2..\lfloor \frac n2 \rfloor]$, the SEMO cannot optimize the $\ojzj_{n,k}$ function.
\label{thm:semoojzj}
\end{theorem}

\begin{proof}
Note that any individual with $m' \in [k..n-k]$ ones has the function value $(m'+k,n-m'+k) \ge (2k,2k)$. Since a search point with $m \in [1..k-1]$ ones has the function value $(m+k,m)\le(2k-1,k-1)$ and a search point with $m \in [1..k-1]$ zeros has the function value $(m,m+k)\le(k-1,2k-1)$, we know that any newly-generated search point with  number of ones in $[k..n-k]$ will lead to the removal from the current population of all individuals with number of zeros or ones in $[1..k-1]$ and will also prevent any such individual from entering the population in the future.

W.l.o.g., suppose that the initial individual has at most $\frac n2$ ones. We show that $1^n$ can never enter the population.
 Assume first that never in this run of the SEMO a search point with $n-k$ ones is generated. Since each offspring has a Hamming distance of one from its parent, the statement ``Never an individual with $n-k$ or more ones is generated''  {holds in each generation.}
 In particular, the search point $1^n$ is never generated. 

Hence assume now that at some time $t$ for the first time an individual $x$ with $n-k$ ones is generated. As in the previous paragraph, up to this point no search point with more than $n-k$ ones is generated. In particular, the point $1^n$ is not generated so far. Since $x$ lies on the Pareto front, it enters the population. From this point on, by our initial considerations, no individual with $n-k+1$ to $n-1$ ones can enter the population. In particular, at no time a parent with $n-1$ ones can be selected, which is a necessary precondition for generating the point~$1^n$. Consequently, the point~$1^n$ will never be generated. 
\end{proof}

\section{Runtime Analysis for the GSEMO}
\label{sec:gsemo}
As we have shown in the previous section, to deal with the \ojzj benchmark a mutation-based algorithm needs to be able to flip more than one bit. The global SEMO (GSEMO), proposed by~\cite{Giel03}, is a well-analyzed MOEA that has this ability. Generalized from the \oea algorithm, it uses the standard bit-wise mutation, that is, each bit is flipped independently with the same probability of, usually, $1/n$. The details are shown in Algorithm~\ref{alg:gsemo}. 
\begin{algorithm}[!ht]
    \caption{GSEMO}
    {\small
    \begin{algorithmic}[1]
    \STATE {Generate $x\in\{0,1\}^n$ uniformly at random and $P\leftarrow \{x\}$}
    \LOOP
    \STATE {Uniformly at random select one individual $x$ from $P$}
    \STATE {Generate $x'$ via independently flipping each bit value of $x$ with probability $1/n$}
    \IF {there is no $y \in P$ such that $x' \preceq y$}
    \STATE {$P=\{z\in P \mid z \npreceq x'\} \cup \{x'\}$}
    \ENDIF
    \ENDLOOP
    \end{algorithmic}
    \label{alg:gsemo}
    }
\end{algorithm}

Theorem~\ref{thm:gsemoojzj} will show that GSEMO can find the Pareto front. To make the main proof clearer, some propositions are extracted as the following lemmas. 
Lemma~\ref{lem:spread} shows that once a solution in the inner part of the Pareto front is part of the population, it takes $O(n^2 \log n)$ iterations until the whole inner part of the Pareto front is covered.

\begin{lemma}
Consider optimizing the function $f := \ojzj_{n,k}$ via the GSEMO. Let $a_0\in[2k..n]$. Assume that at some time the population contains an individual $x$ with $f(x) = (a_0,2k+n-a_0)$. Then the expected number of interations until all $(a,2k+n-a)$, $a\in[2k..n]$, are covered by the population is at most $2en(n-2k+3) (\ln \lceil \frac n2 \rceil+1)$ iterations.
\label{lem:spread}
\end{lemma}
\begin{proof}
Note that any solution corresponding to $(a,2k+n-a)$ in the Pareto front contains exactly $a-k$ ones. If the population contains such an individual, then the probability to find $(a+1,2k+n-a-1)$ in the Pareto front is at least
\begin{align*}
\frac{1}{n-2k+3}\frac{n-a+k}{n}\left(1-\frac1n\right)^{n-1}\ge \frac{n-a+k}{en(n-2k+3)},
\end{align*}
where we use that the population size is at most $n-2k+3$, see Corollary~\ref{cor:population}. Hence, by a simple Markov chain argument (similar to Wegener's fitness level method~\cite{Wegener01}), the expected time to find all $(a,2k+n-a)$, $a\in[a_0+1..n]$, is at most
\begin{align}
\sum_{a=a_0}^{n-1} \frac{en(n-2k+3)}{n-a+k}.
\label{eq:half}
\end{align}

Similarly, the probability to find $(a-1,2k+n-a+1)$ in the Pareto front once the population contains a solution with function value $(a,2k+n-a)$ is at least
\begin{align*}
\frac{1}{n-2k+3}\frac{a-k}{n}\left(1-\frac1n\right)^{n-1}\ge \frac{a-k}{en(n-2k+3)},
\end{align*}
and the expected time to find all $(a,2k+n-a)$, $a\in[2k..a_0-1]$, is at most
\begin{align}
\sum_{a=a_0}^{2k+1} \frac{en(n-2k+3)}{a-k}.
\label{eq:anhalf}
\end{align}

By comparing the individual summands, the sum of (\ref{eq:half}) and (\ref{eq:anhalf}) is at most 
\begin{align*}
2\sum_{a=\lfloor n/2 \rfloor + k}^{n-1} &\frac{en(n-2k+3)}{n-a+k} \le {2en(n-2k+3)\left(\ln \left\lceil \frac n2 \right\rceil +1\right)}.
\qedhere
\end{align*}
\end{proof}

From a population covering the whole inner part of the Pareto front, the two missing extremal points of the front can be found in roughly $O(n^{k+1})$ iterations, as the following Lemma~\ref{lem:jump} shows.

\begin{lemma}
Consider the GSEMO optimizing the $\ojzj_{n,k}$ function. Assume that at some time, all $(a,2k+n-a)$, $a\in[2k..n]$, are  {covered by the current population}. Then the two missing points $(a,2k+n-a)$, $a\in\{k,n+k\}$, of the Pareto front will be found in an expected number of at most $\frac32en^k(n-2k+3)$ additional generations.
\label{lem:jump}
\end{lemma}

\begin{proof}
We pessimistically calculate the time to generate two elements $(a,2k+n-a)$, $a\in\{k,n+k\}$, in the Pareto front ignoring the fact that the current population may contain the corresponding solutions $0^n$ and $1^n$ already. 
Similar to the proof of Lemma~\ref{lem:spread}, we note that the probability to find $(k,k+n)$ in the Pareto front from a solution with function value $(2k,n)$ is at least
\begin{align*}
\frac{1}{n-2k+3} \left(\frac1n\right)^k \left(1-\frac1n\right)^{n-k} \ge \frac{1}{en^k(n-2k+3)}.
\end{align*}
The same estimate holds for the probability to find $(k+n,k)$ in the Pareto front from one solution with  function value $(n,2k)$.
Hence, the expected time to find one of $(k,n+k)$ and $(k+n,k)$ is at most $\frac e2 n^k(n-2k+3)$, and the expected time to find the remaining element of the Pareto front is at most $en^k(n-2k+3)$ additional iterations.
\end{proof}

Now we establish our main result for the GSEMO.

\begin{theorem}
Let $n \in \N_{\ge 2}$ and $k \in [1..\lfloor \frac n2 \rfloor]$. The expected runtime of the GSEMO optimizing the $\ojzj_{n,k}$ function is at most 
\[
e(n-2k+3)(\tfrac32n^k+2n\ln \lceil \tfrac n2 \rceil+3).
\]
For $2 \le k = o(n)$, this bound is $\frac 32 e n^{k+1} \pm o(n^{k+1})$.
\label{thm:gsemoojzj}
\end{theorem}

\begin{proof}
We first consider the time until the population $P$ contains an $x$ such that $(f_1(x),f_2(x))=(a,2k+n-a)$ for some $a\in[2k..n]$. If the initial search point has a number of ones in $[k..n-k]$, then such an $x$ is already found. Otherwise, the initial search point has at most $k-1$ ones or at most $k-1$ zeros. W.l.o.g., we consider an initial point with $m\le k-1$ zeros. If $m=0$, that is, the initial point is~$1^n$, then its function value is $(n+k,k)$. Since any point $\tilde{x}$ with $\tilde{m}\in [1..k-1]$ zeros has function value $(\tilde{m},k+\tilde{m})$, such a point is not dominated by $1^n$, and thus will be kept in the population. After that, since $(\tilde{m},k+\tilde{m})$ increases with respect to $\tilde{m}$, one individual with larger $\tilde{m} \in [1..k-1]$ will replace the one with smaller $\tilde{m} \in [1..k-1]$. Finally, any individual with $m'\in[k..n-k]$ zeros has function value $(n-m'+k,m'+k)>(k-1,2k-1)$ and thus dominates all search points with $\tilde{m} \in [1..k-1]$ zeros. Hence, we pessimistically add up the expected waiting times for increasing the number of zeros one by one until one individual with $k$ zeros is found, which need an expected number of iterations of at most
\begin{equation}
\begin{split}
\sum_{m=0}^{k-1}&\left(\frac{1}{n-2k+3} \frac{n-m}{n} \left(1-\frac{1}{n}\right)^{n-1}\right)^{-1} \le \sum_{m=0}^{k-1} \frac{en(n-2k+3)}{n-m} \\
\le{} &  {en}(n-2k+3)\frac{k}{n-k+1} \le en(n-2k+3).
\end{split}
\label{eq:ingap}
\end{equation}

Then from Lemma~\ref{lem:spread}, we know that all $(a,2k+n-a)$, $a\in[2k..n]$, in the Pareto front will be found in an expected number of at most $2en(n-2k+3) (\ln \lceil \frac n2 \rceil+1)$ iterations. After that, from Lemma~\ref{lem:jump}, the points $(k,n+k)$ and $(n+k,k)$ in the Pareto front will be found in an expected number of at most $\frac32en^k(n-2k+3)$ additional iterations.  
Hence, the expected iterations to find the Pareto front is at most $en(n-2k+3)+ 2en(n-2k+3) (\ln \lceil \frac n2 \rceil+1)+\frac32en^k(n-2k+3)= e(n-2k+3)(\frac32n^k+2n\ln \lceil \frac n2 \rceil+3)$.
%
\end{proof}

 {We finally show that this bound is very tight. For convenience, we only consider the case $k \ge 4$. In this case, it is sufficiently hard to generate an extremal point of the Pareto front from the inner part that we do not need a careful analysis of the population dynamics in the short initial phase in which the inner part of the Pareto front is detected, which is not the main theme of this work. That said, we are very optimistic that it would pose no greater difficulties to show that with high probability at the first time when a solution with less than $\frac 14 n$ or more than $\frac 34 n$ ones enters the population, the population already contains $\Theta(n)$ individuals. Hence from this point on, only every $\Theta(n)$ iterations an outer-most point will be chosen as parent, and this should suffice to prove an $\Omega(n^{k+1})$ lower bound also for $k=2$ and $k=3$.}
\begin{theorem}
Let $n \in \N_{\ge 8}$ and $k \in [4..\frac n2 -1]$. The expected runtime of the GSEMO optimizing the $\ojzj_{n,k}$ function is at least 
\[
\frac 32 e(n-2k+1)n^k\frac{n^k}{(n-1)^k} \left(1-\frac{1}{n^{n/4-2}}-\frac{2}{n^{n-2k}}-\frac{e+3}{n} - \frac{4e (\ln n+1)}{n^{k-3}}\right).
\]
If $k = o(n)$, this bound becomes $\frac 32 e n^{k+1} - o(n^{k+1})$ and matches the upper bound of Theorem~\ref{thm:gsemoojzj} apart from lower-order terms.
\label{thm:lowerbound}
\end{theorem}
\begin{proof}
 {We first prove that with high probability the inner region of the Pareto front will be fully covered before $0^n$ or $1^n$ is generated and then use this to bound the time until the two extremal points of the Pareto front are found.}

For the initial individual $x$, we have $E[|x|_1]=\frac 12n$. With the additive Chernoff bound, see, e.g.,~\cite[Theorem~1.10.7]{Doerr20bookchapter}, we know that with probability at least 
\begin{align*}
1-2\exp\left(-\frac 2n \left(\frac n4\right)^2\right)=1-2\exp\left(-\frac18n\right),
\end{align*}
the initial individual has $\frac14 n < |x|_1 < \frac 34 n$. Then $|x|_1 < k$ or $|x|_1 > n-k$ only if $k > \frac14 n$. If $k>\frac 14 n$, w.l.o.g., let $|x|_1 < k$, then we will show that with probability at least $1-{n^{-n/4+2}}-{2}{n^{-n+2k}}-en^{-1}$, an individual $x'$ with $k\le |x'|_1 \le n-k$ is generated earlier than $1^n$ or $0^n$. Note that before such an $x'$ appears, the population consists of one individual with $\frac 14n < |x|_1 < k$ except if $1^n$, $0^n$, or a $y$ with $|y|_1 > n-k$ is generated. We note that the probability to generate one of the exceptions from one individual with $\frac 14n < |x|_1 < k$ is at most
\begin{align*}
\frac{1}{n^{n-|x|_1}}+\frac{1}{n^{|x|_1}}+\frac{1}{n^{n-k+1-|x|_1}} < \frac{1}{n^{n-k+1}}+\frac{1}{n^{\frac14 n}}+\frac{1}{n^{n-2k+2}} 
\le \frac{1}{n^{\frac14 n}}+\frac{2}{n^{n-2k+2}},
\end{align*}
where the last inequality uses $k\ge 1$. Then with probability at least
\begin{align*}
\left(1-n^{-n/4}-{2}{n^{-(n-2k+2)}}\right)^{n^{2}} \ge 1-{n^{-n/4+2}}-{2}{n^{-n+2k}},
\end{align*} 
the population consists of only one individual with $\frac 14 n< |x|_1 < k$ in $n^{2}$ iterations. Then from the similar calculation in (\ref{eq:ingap}) except changing $n-2k+3$ to $1$, we know that it takes an expected time of at most $en$ iterations to generate such an $x'$. Via Markov's inequality, we know that with probability at least $1-\frac{en}{n^{2}}=1-\frac{e}{n}$, such an $x'$ will be generated in $n^{2}$ iterations. Hence, with probability at least $1-{n^{-n/4+2}}-{2}{n^{-n+2k}}-en^{-1}$, an individual $x'$ with $k\le |x'|_1 \le n-k$ is generated earlier than $1^n$ or $0^n$.

From this point on, we will first show that with probability at least $1-\frac{8e (\ln n+1)}{n^{k-3}}$, all $(a,2k+n-a)$, $a\in[2k..n]$, in the Pareto front are found before $0^n$ or $1^n$ is generated. It is not difficult to see that any search point $y$ in either of the two gaps, that it, with $|y|_1 \in [1..k-1] \cup [n-k+1..n-1]$, has both objective values less than any search point $z$ with $|z|_1\in [k..n-k]$, and thus cannot enter into the population. Therefore, the probability to generate $0^n$ or $1^n$ from some parent $x$ with $|x|_1 \in [k..n-k]$ is at most
\begin{align}
\frac{1}{n^{|x|_1}}\left(1-\frac{1}{n}\right)^{n-|x|_1}+\frac{1}{n^{n-|x|_1}}\left(1-\frac{1}{n}\right)^{|x|_1} \le \frac{2}{n^k}.
\label{eq:twoends}
\end{align}
From Lemma~\ref{lem:spread}, we know that all $(a,2k+n-a)$, $a\in[2k..n]$, in the Pareto front will be found in an expected number of at most $2en(n-2k+3) (\ln \lceil \frac n2 \rceil+1)$ iterations. Via Markov's inequality, we know that with probability at least $1-\frac{2en(n-2k+3) (\ln \lceil \frac n2 \rceil+1)}{2en^3 (\ln n+1)} \ge1-\frac{1}{n}$, all $(a,2k+n-a)$, $a\in[2k..n]$, in the Pareto front will be found in $2en^3 (\ln n+1)$ iterations. Hence, with (\ref{eq:twoends}), we know that the event that during all $2en^3 (\ln n+1)$ iterations, $0^n$ or $1^n$ is not generated,  {has probability} at least
\begin{align*}
\left(1-\frac{2}{n^k}\right)^{2en^3 (\ln n+1)} \ge 1-\frac{4en^3 (\ln n+1)}{n^k}=1-\frac{4e (\ln n+1)}{n^{k-3}}.
\end{align*}

Now all $(a,2k+n-a)$, $a\in[2k..n]$, in the Pareto front are found and the current population has size $n-2k+1$. From this time on, we compute the probability to generate the remaining Pareto optimal solution $0^n$ and $1^n$ as
\begin{align*}
\sum_{i=k}^{n-k}&\frac{1}{n-2k+1}\left(\frac{1}{n^i}\left(1-\frac 1n\right)^{n-i}+\frac{1}{n^{n-i}}\left(1-\frac 1n\right)^{i}\right)\\
= {} &  {\frac{2}{(n-2k+1)n^k}}\left(1-\frac{1}{n}\right)^{n-k}\left(1+\sum_{j=1}^{n-2k} \frac{1}{n^j}\left(1-\frac1n\right)^{-j}\right)\\
\le{}&  {\frac{2}{(n-2k+1)n^k}}\left(1-\frac{1}{n}\right)^{n-k} \left(1+\frac{\frac1{n-1}}{1-\frac1{n-1}}\right) \\
={} &  {\left(1+\frac{1}{n-2}\right)}\frac{2}{(n-2k+1)n^k} \left(1-\frac{1}{n}\right)^{n-k}
=\frac{n-1}{n-2}\frac{2}{(n-2k+1)n^k} \left(1-\frac{1}{n}\right)^{n-k}.
\end{align*}
Then the expected time to find $0^n$ or $1^n$ is at least $\frac{1}{2}(n-2k+1)n^k\left(1-\frac{1}{n}\right)^{k-n}\frac{n-2}{n-1}$. W.l.o.g., $0^n$ is found. Now the current population size is $n-2k+2$, and the probability to find the remaining $1^n$ is
\begin{align*}
\frac{1}{n-2k+2}&\cdot \frac{1}{n^n}+\sum_{i=k}^{n-k}\frac{1}{n-2k+2}\cdot\frac{1}{n^i}\left(1-\frac 1n\right)^{n-i}\\
\le{} &  {\frac{1}{(n-2k+2)n^k}}\left(1-\frac{1}{n}\right)^{n-k}\left(\frac{1}{n^{n-k}}\left(1-\frac{1}{n}\right)^{-n+k} +1+\frac{1}{n-2}\right)\\
\le{} &  {\frac{1}{(n-2k+2)n^k}}\left(1-\frac{1}{n}\right)^{n-k}\left(1+\frac{2}{n-2}\right)\\
={} &   {\frac{1}{(n-2k+2)n^k}}\left(1-\frac{1}{n}\right)^{n-k} \frac{n}{n-2}
\end{align*}
Then we need at least $(n-2k+2)n^k\left(1-\frac{1}{n}\right)^{k-n}\frac{n-2}{n}$ additional iterations in expectation to cover the whole Pareto front.

In summary, together with the above discussed probability for the initial individual and the probability that all $(a,2k+n-a)$, $a\in[2k..n]$, in the Pareto front are found before $0^n$ or $1^n$ is generated, we obtain that the lower bound to cover the Pareto front is 
\begin{align*}
\bigg(1-&\frac{1}{n^{n/4-2}}-\frac{2}{n^{n-2k}}-\frac{e}{n} - \frac{1}{n} - \frac{4e (\ln n+1)}{n^{k-3}}\bigg) 
 \cdot \frac 32(n-2k+1)n^k\left(1-\frac{1}{n}\right)^{k-n}\frac{n-2}{n}\\
\ge{} &  {\frac 32}(n-2k+1)n^k\left(1-\frac{1}{n}\right)^{k-n} \left(1-\frac{1}{n^{n/4-2}}-\frac{2}{n^{n-2k}}-\frac{e+3}{n} - \frac{4e (\ln n+1)}{n^{k-3}}\right)\\
\ge{} &  {\frac 32} e(n-2k+1)n^k\frac{n^k}{(n-1)^k} \left(1-\frac{1}{n^{n/4-2}}-\frac{2}{n^{n-2k}}-\frac{e+3}{n} - \frac{4e (\ln n+1)}{n^{k-3}}\right).
\qedhere
\end{align*}
\end{proof}

\section{GSEMO with Heavy-Tailed Mutation}
\label{sec:heavytail}

In the previous section, we have shown that the GSEMO can optimize our multimodal optimization problem, but similar to the single-objective world (say, the optimization of \jump functions via simple evolutionary algorithms~\cite{DrosteJW02}), the runtime increases significantly with the distance $k$ that a solution on the Pareto front can have from all other solutions on the front. As has been discussed in~\cite{DoerrLMN17}, increasing the mutation rate can improve the time it takes to jump over such gaps. However, this work also showed that a deviation from the optimal mutation rate can be costly: A small constant-factor deviation from the optimal rate $k/n$ leads to a performance loss exponential in $k$. For this reason, a heavy-tailed mutation operator was proposed. Compared to using the optimal (usually unknown) rate, it only loses a small polynomial factor (in $k$) in performance. 

We now equip the GSEMO with the heavy-tailed mutation operator from~\cite{DoerrLMN17} and observe similar advantages. We start by introducing the heavy-tailed mutation operator and giving a first non-asymptotic estimate for the probabilities to flip a certain number of bits.

\subsection{Heavy-Tailed Mutation Operator}

The following capped power-law distribution will be used in the definition of the heavy-tailed mutation operator.

\begin{definition}[Power-law distribution $D_{n/2}^{\beta}$]
Let $n\in \N_{\ge 2}$ and $\beta > 1$. Let $C_{n/2}^{\beta} :=\sum_{i=1}^{n/2} i^{-\beta}$. We say a random variable $\xi$ follows the power-low $D_{n/2}^{\beta}$, written as $\xi \sim D_{n/2}^{\beta}$, if for all $\alpha \in [1..n/2]$, we have
\begin{align*}
\Pr[\xi = \alpha] = \left(C_{n/2}^{\beta}\right)^{-1}\alpha^{-\beta}.
\end{align*}
\end{definition}

The heavy-tailed mutation operator proposed by~\cite{DoerrLMN17}, in the remainder denoted by $\mut^{\beta}(\cdot)$, in each application independently samples a number $\alpha$ from the power-law distribution $D_{n/2}^{\beta}$ and then uses standard bit-wise mutation with mutation rate ${\alpha}/{n}$, that is, flips each bit independently with probability $\alpha/n$. 
 \cite{DoerrLMN17} shows the two following properties of $\mut^{\beta}$.
\begin{lemma}[\cite{DoerrLMN17}]
Let $x\in\{0,1\}^n$ and $y\sim \mut^{\beta}(x)$. Let $H(x,y)$ denote the Hamming distance between $x$ and $y$. Then we have
\begin{align*}
P_i^{\beta}:=\Pr[H(x,y)=i] & = 
\begin{cases}
\left(C_{n/2}^{\beta}\right)^{-1} \Theta(1), & \text{for}~i=1;\\
\left(C_{n/2}^{\beta}\right)^{-1} \Omega(i^{-\beta}), & \text{for}~i=[2..n/2].
\end{cases}
\end{align*}
\label{lem:htpro}
\end{lemma}

In order to obtain an understanding of the leading constants when generating a solution with a not too large Hamming distance, we prove the following estimations, which have not been shown before, and can also be applied to other previous works about the theoretical analysis of the heavy-tailed mutation.  {We suspect tighter estimates of the leading constants exist but leave it as interesting future work.} 
\begin{lemma}
Let $x\in\{0,1\}^n$ and $y\sim \mut^{\beta}(x)$. Let $H(x,y)$ denote the Hamming distance between $x$ and $y$. Then we have
\begin{align*}
P_i^{\beta}:=\Pr[H(x,y)=i] & \ge 
\begin{cases}
\frac{\beta-1}{e\beta}, & \text{for}~i=1;\\
\tfrac{\beta-1}{ {2}\sqrt{2\pi}e^{8\sqrt 2 +13}\beta}i^{-\beta}, & \text{for}~i=[2..\lfloor \sqrt n \rfloor].
\end{cases}
\end{align*}
\label{lem:htp}
\end{lemma}
 {The proof of Lemma~\ref{lem:htp} will use an elementary estimate for $\binom{n}{k}$ when $k\le \sqrt n$. Since we have not found a good reference for it and since we believe that it might be useful for other runtime analyses, we formulate it as a separate lemma and give a formal proof.
\begin{lemma}
Let $n\in \N$ and $i=[0..\lfloor\sqrt n\rfloor]$. Then $\binom{n}{i} \ge \frac{n^i}{2i!}$.
\label{lem:bino}
\end{lemma}
\begin{proof}
Our claim holds trivially for $i=0$ and $1$. Thus, for any $n\in[1..3]$, this lemma holds. We now discuss $n\ge 4$ and $i\in[2..\lfloor\sqrt n\rfloor]$. By Weierstrass product inequality, see, e.g., \cite[Lemma~1.4.8(b)]{Doerr20bookchapter}, we know 
\begin{align*}
\left(1-\frac{1}{n}\right)\left(1-\frac2n\right)\dots\left(1-\frac{i-1}{n}\right) \ge 1-\sum_{j=1}^{i-1}\frac jn=1-\frac{\frac{(i-1)(i-1+1)}{2}}{n} \ge 1-\frac{i^2}{2n} \ge \frac12.
\end{align*}
Then we have
\begin{align*}
\binom{n}{i}&{}={}\frac{n!}{(n-i)!i!}=\frac{n(n-1)\dots(n-i+1)}{i!} =\frac{n^i}{i!}\frac{n(n-1)\dots(n-i+1)}{n^i}\\
&{}={}\frac{n^i}{i!}\left(1-\frac{1}{n}\right)\left(1-\frac2n\right)\dots\left(1-\frac{i-1}{n}\right) \ge \frac{n^i}{2i!}.
\qedhere
\end{align*}
\end{proof}}

 {Now we prove Lemma~\ref{lem:htp}.}
\begin{proof}[Proof of Lemma~\ref{lem:htp}]
We calculate
\begin{align*}
P_1^{\beta}=\Pr[H(x,y)=1] \ge \left(C_{n/2}^{\beta}\right)^{-1} \frac{1}{1^{\beta}}\binom{n}{1}\frac1n\left(1-\frac1n\right)^{n-1} \ge \frac{1}{eC_{n/2}^{\beta}}.
\end{align*}
 {Since $\beta>1$, we know 
\begin{equation}
\begin{split}
C_{n/2}^{\beta}={}&\sum_{i=1}^{n/2} i^{-\beta}\le1+ \int_{1}^{n/2}x^{-\beta}\,dx 
=1+\frac{(n/2)^{-\beta+1}-1}{-\beta+1}\le 1+\frac{1}{\beta-1}=\frac{\beta}{\beta-1}.
\end{split}
\label{eq:p1}
\end{equation}
Hence, $P_1^{\beta} \ge \frac{\beta-1}{e\beta}$.}

For $i\in[2..\lfloor \sqrt n \rfloor]$, we have 
\begin{align*}
\binom{n}{i} \ge \frac{n^i}{ {2}i!} \ge \frac{n^i}{ {2}\sqrt{2\pi i} (i/e)^i e^{1/(12i)}} =\frac{1}{ {2}\sqrt{2\pi i}} \left(\frac{n}{i}\right)^{i}e^{i-1/(12i)},
\end{align*}
where the second inequality uses Stirling's approximation. Using that $j \rightarrow j^i(1-\tfrac{j}{n})^{n-i}$ is unimodal and has its maximum at $j=i$, we compute 
\begin{equation*}
\begin{split}
P_i^{\beta}={}&\left(C_{n/2}^{\beta}\right)^{-1}\sum_{j=1}^{n/2}  \frac{1}{j^{\beta}}\binom{n}{i}\left(\frac{j}{n}\right)^i\left(1-\frac{j}{n}\right)^{n-i} \\
\ge {}& \frac{e^{i-1/(12i)}}{ {2}\sqrt{2\pi i} C_{n/2}^{\beta} } \sum_{j=1}^{n/2}  \frac{1}{j^{\beta}}\left(\frac{j}{i}\right)^i\left(1-\frac{j}{n}\right)^{n-i}\\
\ge {}&  \frac{ e^{i-1/(12i)}}{ {2}\sqrt{2\pi i} C_{n/2}^{\beta}} \sum_{j=i-\lfloor \sqrt i \rfloor}^{i}  \frac{1}{j^{\beta}}\left(\frac{j}{i}\right)^i\left(1-\frac{j}{n}\right)^{n-i}\\
\ge {} &  \frac{e^{i-1/(12i)}}{ {2}\sqrt{2\pi i} C_{n/2}^{\beta} \cdot i^{\beta}}  \sum_{j=i-\lfloor \sqrt i \rfloor}^{i}  \left(\frac{j}{i}\right)^i\left(1-\frac{j}{n}\right)^{n-i} \\
\ge {} &  \frac{e^{i-1/(12i)} (\lfloor \sqrt i \rfloor +1)}{ {2}\sqrt{2\pi i} C_{n/2}^{\beta} \cdot i^{\beta}}  \left(\frac{i-\sqrt i}{i}\right)^i\left(1-\frac{i-\sqrt i}{n}\right)^{n-i}\\
\ge {} &  \frac{1 }{ {2}\sqrt{2\pi} C_{n/2}^{\beta} i^{\beta}} \exp\left(i-\frac{1}{12i}\right) \exp\left(-\frac{i}{\sqrt i - 1}\right) \exp\left(-\frac{n-i}{\frac{n}{i-\sqrt i}-1}\right)\\
 = {} & {} \frac{1}{ {2}\sqrt{2\pi} C_{n/2}^{\beta} i^{\beta}} \exp\left(i-\frac{1}{12i}-\frac{i}{\sqrt i - 1}-\frac{n-i}{\frac{n}{i-\sqrt i}-1}\right).
\end{split}
\end{equation*}
Now we calculate
\begin{equation*}
\begin{split}
i-\frac{1}{12i}-{} & \frac{i}{\sqrt i - 1}-\frac{n-i}{\frac{n}{i-\sqrt i}-1} = -\frac{1}{12i}-\frac{i}{\sqrt i - 1}+\frac{n\sqrt i}{n-i+\sqrt i}\\
= {} & -\frac{1}{12i}-\frac{n\sqrt i - i^2 +i\sqrt i}{(n-i+\sqrt i)(\sqrt i - 1)}
\ge  -\frac{1}{24}-\frac{n\sqrt i+i \sqrt i}{(n-i)\sqrt i /4} \\
= {} &-\tfrac{1}{24}-4\left(\tfrac{2n}{n-i}-1\right) \ge -\tfrac{1}{24}-4\left(\tfrac{2n}{n-\sqrt n}-1\right)  \\
=  {} & -\tfrac{1}{24}-4\left(\tfrac{2}{\sqrt n-1}+1\right) \ge -\tfrac{1}{24}-4\left(\tfrac{2}{\sqrt 2-1}+1\right) \\
\ge {} & -8\sqrt 2 -13,
\end{split}
\end{equation*}
where the first inequality uses that $n\sqrt i - i^2 +i\sqrt i \ge 0$ for $i < \sqrt n$ and $\sqrt i - 1 \ge \sqrt i / 4$ for $i \ge 2$, the second inequality uses $i < \sqrt n$, and the third inequality uses $n\ge 2$. Then we know 
\begin{equation*}
P_i^{\beta} \ge  \frac{\left(C_{n/2}^{\beta} \right)^{-1}  }{ {2}\sqrt{2\pi} i^{\beta}} e^{-8\sqrt 2 -13}  {~\ge \frac{\beta-1}{ {2}\sqrt{2\pi}e^{8\sqrt 2 +13}\beta}i^{-\beta},}
\end{equation*}
 {where the last inequality uses (\ref{eq:p1}). Hence,} this lemma is proven.
\end{proof}

Equipping the standard GSEMO with this mutation operator $\mut^{\beta}$ gives Algorithm~\ref{alg:gsemohtm}, which we call GSEMO-HTM. 
\begin{algorithm}[!ht]
    \caption{The GSEMO-HTM algorithm with power-law exponent $\beta > 1$}
    {\small
    \begin{algorithmic}[1]
    \STATE {Generate $x\in\{0,1\}^n$ uniformly at random, and $P\leftarrow \{x\}$}
    \LOOP
    \STATE {Uniformly at random select one individual $x$ from $P$}
    \STATE {Sample $\alpha$ from $D_{n/2}^{\beta}$ and generate $x'$ via independently flipping each bit value of $x$ with probability $\alpha/n$}
    \IF {there is no $y \in P$ such that $x' \preceq y$}
    \STATE {$P=\{z\in P \mid z \npreceq x'\} \cup \{x'\}$}
    \ENDIF
    \ENDLOOP
    \end{algorithmic}
    \label{alg:gsemohtm}
    }
\end{algorithm}

\subsection{GSEMO-HTM on $\ojzj_{n,k}$}
We now analyze the runtime of the GSEMO-HTM on $\ojzj_{n,k}$. We start with an estimate for the time it takes to find the inner part of the Pareto front. 

\begin{lemma}
Consider the GSEMO-HTM optimizing the $\ojzj_{n,k}$ function. Let $a_0\in[2k..n]$. If in a certain iteration, the point $(a_0,2k+n-a_0)$ is  {covered by the current population}, then all $(a,2k+n-a)$, $a\in[2k..n]$, in the Pareto front will be found in an expected number of at most $\frac{2}{P_1^{\beta}}n(n-2k+3) (\ln \lceil \frac n2 \rceil+1)$ iterations.
\label{lem:spreadhtm}
\end{lemma}
\begin{proof}
The proof is very similar to the one of Lemma~\ref{lem:spread}. As there, we ignore any positive effect of mutation flipping more than one bit. Noting that the corresponding solution for $(a,2k+n-a)$ in the Pareto front has $a-k$ ones, the probability to find $(a+1,2k+n-a-1)$ in the Pareto front conditional on the population has one solution with  function value $(a,2k+n-a)$ is at least
\begin{align*}
\frac{1}{n-2k+3}\frac{n-a+k}{n} P_1^{\beta} = \frac{(n-a+k)P_1^{\beta}}{n(n-2k+3)}.
\end{align*}
Hence, the expected time to find all $(a,2k+n-a)$ for $a\in[a_0..n]$ is at most
\begin{align}
\sum_{a=a_0}^{n-1} \frac{n(n-2k+3)}{(n-a+k)P_1^{\beta}}.
\label{eq:innerhtm}
\end{align}

Similarly, the probability to find $(a-1,2k+n-a+1)$ in the Pareto front conditional on the population containing a solution with function value $(a,2k+n-a)$ is at least
\begin{align*}
\frac{1}{n-2k+3}\frac{a-k}{n} P_1^{\beta} = \frac{(a-k) P_1^{\beta}}{n(n-2k+3)},
\end{align*}
and the expected time to find all $(a,2k+n-a)$ for $a\in[2k..a_0]$ is at most
\begin{align}
\sum_{a=2k+1}^{a_0} \frac{n(n-2k+3)}{(a-k)P_1^{\beta}}.
\label{eq:aninnerhtm}
\end{align}

By comparing the individual summands, the sum of (\ref{eq:innerhtm}) and (\ref{eq:aninnerhtm}) is at most 
\begin{align*}
2\sum_{a=\lfloor n/2 \rfloor + k}^{n-1} &\frac{n(n-2k+3)}{(n-a+k)P_1^{\beta}} \le \frac{2n(n-2k+3)(\ln \lceil \frac n2 \rceil +1)}{P_1^{\beta}}.
\qedhere
\end{align*}
\end{proof}

 {We now estimate the time to find the two extreme Pareto front points after the coverage of the inner part of the Pareto front.}
\begin{lemma}
Consider the GSEMO-HTM optimizing the $\ojzj_{n,k}$ function. Assume that at some time, all $(a,2k+n-a)$, $a\in[2k..n]$, are  {covered by the current population}. Then the two missing $(a,2k+n-a)$, $a\in\{k,n+k\}$, in the Pareto front will be found in an expected number of at most $\frac{3}{2P_k^{\beta}}\binom{n}{k}(n-2k+3)$ additional generations.
\label{lem:jumphtm}
\end{lemma}
\begin{proof}
We note that the probability to find $(k,k+n)$ in the Pareto front from one solution with function value $(2k,n)$, or to find $(k+n,k)$ in the Pareto front from one solution with  function value $(n,2k)$ is at least
\begin{align*}
\frac{2}{n-2k+3} \binom{n}{k}^{-1} P_k^{\beta}.
\end{align*}
Hence, the expected time to find $(k,n+k)$ or $(k+n,k)$ is at most $\binom{n}{k}(n-2k+3)/(2P_k^{\beta})$. 

W.l.o.g., $(k,n+k)$ is found  {first}. Then the probability to find $(k+n,k)$ in the Pareto front from one solution with function value $(n,2k)$ is at least
\begin{align*}
\frac{1}{n-2k+3} \binom{n}{k}^{-1} P_k^{\beta}.
\end{align*}
Hence, the expected additional time to find the last missing element in the Pareto front is at most $\binom{n}{k}(n-2k+3)/P_k^{\beta}$. Then this lemma is proven.
\end{proof}

Now we are ready to show the runtime for the GSEMO-HTM on $\ojzj_{n,k}$.
\begin{theorem}
The expected runtime of the GSEMO-HTM optimizing the $\ojzj_{n,k}$ function is at most $(n-2k+3)O(k^{\beta-0.5})C_{n/2}^{\beta} \frac{n^n}{k^k (n-k)^{n-k}}$.
\label{thm:gsemohojzj}
\end{theorem}
\begin{proof}
We first consider the time to find one $(a,2k+n-a)$ for some $a\in[2k..n]$ in the Pareto front. If the initial search point has a number of ones in $[k,n-k]$, then $(a,2k+n-a)$ for some $a\in[2k..n]$ in the Pareto front is already found. Otherwise, we have the initial search point with at most $k-1$ ones or zeros. W.l.o.g., we consider the initial point with $m\le k-1$ zeros. Similar to the discussion in the proof of Theorem~\ref{thm:gsemoojzj}, we pessimistically add the waiting time to increase the number of zeros one by one until one individual with $k$ zeros is found. This gives the expected number of iterations of at most
\begin{align*}
\sum_{m=0}^{k-1}\left(\frac{1}{n-2k+3} \frac{n-m}{n} P_1^{\beta} \right)^{-1} \le \sum_{m=0}^{k-1} \frac{n(n-2k+3)}{(n-m)P_1^{\beta}} 
\le \frac{n(n-2k+3)(\ln n +1)}{P_1^{\beta}}.
\end{align*}

Then from Lemma~\ref{lem:spreadhtm}, we know that all $(a,2k+n-a)$ for $a\in[2k..n]$ in the Pareto front will be found in at most $\frac{2}{P_1^{\beta}}n(n-2k+3) (\ln (n-k)+1)$ iterations in expectation. After that, from Lemma~\ref{lem:jumphtm}, the points $(k,n+k)$ and $(n+k,k)$ enter the Pareto front in at most $\frac{2}{P_k^{\beta}}\binom{n}{k}(n-2k+3)$ iterations in expectation.  

Hence, the expected time to cover the Pareto front is at most
\begin{equation}
\begin{split}
(n-2k+3)& \left( \frac{2n (\ln \lceil \frac n2 \rceil+1)}{P_1^{\beta}} + \frac{2}{P_k^{\beta}}\binom{n}{k} + \frac{n(\ln n +1)}{P_1^{\beta}}\right)\\
={} & {(n-2k+3)} \bigg( O(1)C_{n/2}^{\beta}n (\ln n +1)
+O(k^{\beta-0.5})C_{n/2}^{\beta} \frac{n^n}{k^k (n-k)^{n-k}} \bigg)\\
={}& {(n-2k+3)}O(k^{\beta-0.5})C_{n/2}^{\beta} \frac{n^n}{k^k (n-k)^{n-k}},
\end{split}
\label{eq:htmruntime}
\end{equation}
where the first equality uses Lemma~\ref{lem:htpro}.
%
\end{proof}

To give a non-asymptotic runtime bound, by Lemma~\ref{lem:htp}, for $k\in[2..\lfloor \sqrt n \rfloor]$, we know that the first line of (\ref{eq:htmruntime}) can be calculated as
\begin{equation*}
\begin{split}
(n-2k&+3) \left( \frac{2n (\ln \lceil \frac n2 \rceil+1)}{P_1^{\beta}} + \frac{2}{P_k^{\beta}}\binom{n}{k} + \frac{n(\ln n +1)}{P_1^{\beta}}\right)\\
\le{} &  {(n-2k+3)} \left( \frac{e\beta}{\beta-1} 3n (\ln n+1) + \frac{ {4}\sqrt{2\pi}e^{8\sqrt 2 +13}\beta}{\beta-1}k^{\beta} \frac{n!}{k!(n-k)!}\right).
\end{split}
\label{eq:htmruntimeupper}
\end{equation*}
Via the Stirling’s approximation $\sqrt{2\pi} n^{n+0.5}e^{-n} \le n! \le en^{n+0.5}e^{-n}$ and due to the fact that $n/(n-k) \le 2$, we know that 
\begin{align*}
\frac{n!}{k!(n-k)!} \le {} & \frac{en^{n+0.5}e^{-n}}{\left(\sqrt{2\pi} k^{k+0.5}e^{-k}\right) \left(\sqrt{2\pi} (n-k)^{n-k+0.5}e^{-(n-k)}\right)} \\
\le {} & \frac{en^n}{\sqrt 2 \pi k^{k+0.5}(n-k)^{n-k}}.
\end{align*} 
Hence, we easily obtain the following runtime for $k\in[2..\lfloor \sqrt n \rfloor]$.
\begin{theorem}
Let $n\in \N$ and $k\in[2..\lfloor \sqrt n \rfloor]$. Then the expected runtime of the GSEMO-HTM optimizing the $\ojzj_{n,k}$ function is at most $(n-2k+3) (\tfrac{ {4}e^{8\sqrt 2 +14}\beta}{(\beta-1)\sqrt \pi} k^{\beta-0.5}\frac{n^n}{k^k (n-k)^{n-k}} + \tfrac{3e\beta}{\beta-1} n (\ln n+1) )$.
\label{cor:gsemohojzj}
\end{theorem}

Comparing Theorem~\ref{thm:gsemoojzj} and Theorem~\ref{thm:gsemohojzj} ( {Theorem}~\ref{cor:gsemohojzj}), we see that the asymptotic expected runtime of the GSEMO-HTM on $\ojzj_{n,k}$ is smaller than that of the GSEMO by a factor of around $k^{k+0.5-\beta}/e^k$.

\section{GSEMO with Stagnation Detection}
\label{sec:sdgsemo}

In this section, we discuss the non-trivial question of how to adapt the stagnation-detection strategy proposed by~\cite{RajabiW22} to multiobjective optimization, which increases the mutation rate with the time that no improvement has been found. We define a stagnation-detection variant of the GSEMO that has a slightly better asymptotic performance on \ojzj than the GSEMO with heavy-tailed mutation. In contrast to this positive result on \ojzj, we speculate that this algorithm may have difficulties with plateaus of constant fitness.

\subsection{The Stagnation-Detection Strategy of Rajabi and Witt}

\cite{RajabiW22} proposed the following strategy to adjust the mutation rate during the run of the \oea. We recall that the \oea is a very elementary EA working with a parent and offspring population size of one, that is, it generates in each iteration one offspring from the unique parent via mutation and accepts it if it is at least as good as the parent. The classic mutation operator for this algorithm is standard bit-wise mutation with mutation rate $1/n$, that is, the offspring is generated by flipping each bit of the parent independently with probability~$1/n$. 

The main idea of the stagnation-detection approach is as follows. Assume that the \oea for a longer time, say at least $10 n \ln n$ iterations, has not accepted any new solution. Then, with high probability, it has generated (and rejected) all Hamming neighbors of the parent. Consequently, there is no use to generate these solutions again and the algorithm should better concentrate on solutions further away from the parent. This can be achieved by increasing the mutation rate. For example, with a mutation rate of $\frac 2n$ the rate of Hamming neighbors produced is already significantly reduced; however, Hamming neighbors can still be generated, which is important in case we were unlucky so far and missed one of them. 

More generally and more precisely, to implement this stagnation-detection approach the \oea maintains a counter (``failure counter'') that keeps track of how long the parent individual has not given rise to a better offspring. This counter determines the current mutation rate. This dependency is governed by a safety parameter $R$ which is recommended to be at least~$n$. Then for $r = 1, 2, \dots$ in this order the mutation rate $r/n$ is used for 
\begin{equation}\label{eq:tr}
  T_r := \lceil 2 (\tfrac{en}{r})^r \ln(nR) \rceil
\end{equation}
iterations; the rate $\lfloor n/2 \rfloor$ is used until an improvement is found, that is, the mutation rate is never increased above  {$1/2$}. When a strictly improving solution is found, the counter is reset to zero, and consequently, the mutation rate starts again at~$\frac 1n$. 

\cite{RajabiW22} showed that the \oea with this strategy optimizes $\jump_{n,k}$ with $k = o(n)$ in time $\Omega((\frac{en}{k})^k (1 - \frac{k^2}{n-k}))$ and $O((\frac{en}{k})^k)$. In particular, for $k = o(\sqrt n)$, a tight (apart from constant factors independent of $k$ and $n$) bound of $\Theta((\frac{en}{k})^k)$ is obtained. This is faster than the runtime of $\Theta(k^{\beta-0.5} (\frac{en}{k})^k)$ proven by \cite{DoerrLMN17} for the \oea with heavy-tailed mutation with power-law exponent $\beta>1$ by a factor of $k^{\beta-0.5}$. For the recommended choice $\beta = 1.5$, this factor is $\Theta(k)$.  

\subsection{Adaptation of the Stagnation-Detection Strategy to Multiobjective Optimization}

The \oea being an algorithm without a real population, it is clear that certain adaptations are necessary to use the stagnation-detection strategy in multiobjective optimization. 

\subsubsection{Global or Individual Failure Counters}

The first question is how to count the number of unsuccessful iterations. The following two obvious alternatives exist. 

\emph{Individual counters: }From the basic idea of the stagnation-detection strategy, the most natural solution is to equip each individual with its own counter. Whenever an individual is chosen as parent in the GSEMO, its counter is increased by one. New solutions (but see the following subsection for an important technicality of what ``new'' shall mean) entering the population (as well as the random initial individual) start with a counter value of zero. 

\emph{A global counter: }Algorithmically simpler is the approach to use only one global counter. This counter is increased in each iteration. When a new solution enters the population, the global counter is reset to zero.

We suspect that for many problems, both ways of counting give similar results. The global counter appears to be wasteful in the sense that when a new individual enters the population, also parents that are contained in the population for a long time re-start generating offspring with mutation rate $\frac 1n$ despite the fact that they have, with very high probability, already generated as offspring all solutions close by. On the other hand, often these ``old individuals'' do not generate solutions that enter the population anyway, so that {an optimized choice of the mutation rate is less important}. 

For the \ojzj problem, it is quite clear that this second effect is dominant. A typical run starts with some individual in the  {inner} region of the Pareto front. In a relatively short time, the whole middle region is covered, and for this it suffices that relatively recent solutions generate a suitable Hamming neighbor as offspring. The runtime is dominated by the time to find the two extremal solutions and this will almost always happen from the closest parent in the middle region of the front (regardless of whether individual counters or a global counter is used). For this reason, we analyze in the following the simpler approach using a global counter.

\subsubsection{Dealing with Indifferent Solutions}

One question that becomes critical when using stagnation detection is how to deal with indifferent solutions, that is, which solution to put or keep in the population in the case that an offspring $y$ has the same (multiobjective) fitness as an individual $x$ already in the population. Since $f(x) = f(y)$, we have $x \preceq y$ and $y \preceq x$, that is, both solutions do an equally good job in dominating others and thus in approximating the Pareto front. In early works, e.g.~\cite{LaumannsTZWD02} proposing the SEMO algorithm, such later generated indifferent solutions do not enter the population. This is partially justified by the fact that in many of the problems regarded in these works, search points with equal fitness are fully equivalent for the future run of the algorithm. We note that our $\ojzj$ problem also has this property, hence all results presented so far are valid regardless of how indifferent solutions are treated.

When non-equivalent search points with equal fitness exist, it is less obvious how to deal with indifferent solutions. In particular, it is clear that larger plateaus of constant fitness can be traversed much easier when a new indifferent solution is accepted as this allows to imitate a random walk behavior on the plateau~\cite{BrockhoffFHKNZ07}. For that reason, and in analogy to single-objective optimization~\cite{JansenW01}, it seems generally more appropriate to let a new indifferent solution enter the population, and this is also what most of the later works on the SEMO and GSEMO algorithm do~\cite{FriedrichHN10,FriedrichHN11,QianTZ16,LiZZZ16,BianQT18ijcaigeneral,OsunaGNS20}. 

Unfortunately,  {as mentioned in Section~\ref{sec:int},} it is not so clear how to handle indifferent solutions together with stagnation detection. In principle, when a new solution enters the population, the failure counter has to be reset to zero to reset the mutation rate to $1/n$. Otherwise, the continued use of a high mutation rate would prohibit finding good solutions in the direct neighborhood of the new solution. However, the acceptance of indifferent solutions can also lead to unwanted resets. For the \ojzj problem, for example, it is easy to see by mathematical means that in a typical run, it will happen very frequently that an indifferent solution is generated. If this enters the population with a reset of a global failure counter (or an individual counter), then the regular resets will prevent the counters to reach interesting values. In a quick experiment for $n=50$, $k=4$, and a global counter, the largest counter value ever reached in this run of over 500,000,000 iterations was $5$. Consequently, this SD-GSEMO was far from ever increasing the mutation rate and just imitated the classic GSEMO. 

For this reason, in this work we regard the GSEMO with stagnation detection only in the variant that does not accept indifferent solutions, and we take note of the fact that thus our positive results on the stagnation-detection mechanism will not take over to problems with non-trivial plateaus of constant fitness.

\subsubsection{Adjusting the Self-Adjustment}

In the \oea with stagnation detection,~\cite{RajabiW22} increased the mutation rate from $\frac rn$ to $\frac{r+1}n$ once the rate $\frac rn$ has been used for $T_r$ iterations with $T_r$ as defined in~\eqref{eq:tr}. This choice ensured that any particular target solution in Hamming distance $r$ is found in this phase with probability at least $1 - (nR)^{-2}$, see the proof of Lemma~2 in~\cite{RajabiW22}. Since in a run of the GSEMO with current population size $|P|$ each member of the population is chosen as parent only an expected number of $T_r / |P|$ times in a time interval of length $T_r$, we need to adjust the parameter $T_r$. Not surprisingly, by
taking 
\begin{equation}\label{eq:trprime}
\tilde T_r = \lceil 2 \, |P| \,(\tfrac{en}{r})^r \ln(nR) \rceil,
\end{equation}
that is, roughly $|P| \, T_r$,  the probability to generate any particular solution in Hamming distance $r$ in phase $r$ is at least 
\begin{align*}
1-\bigg(&1-\frac{1}{|P|}\left(\frac{r}{n}\right)^r \left(1-\frac{r}{n}\right)^{n-r}\bigg)^{2|P|(en)^r\ln(nR)/r^r} \\
&\ge 1 - \frac{1}{(nR)^2},
\end{align*}
which is sufficient for this purpose. Note that the population size $|P|$ changes only if a new solution enters the population and in this case the mutation rate is reset to $\frac 1n$. Hence the definition of $\tilde T_r$, with the convention that we suppress $|P|$ in the notation to ease reading, is unambiguous.

\subsubsection{The GSEMO with Stagnation Detection: SD-GSEMO}

Putting the design choices discussed so far together, we obtain the following variant of the GSEMO, called SD-GSEMO. Its pseudocode is shown in Algorithm~\ref{alg:sdgsemo}.

\begin{algorithm}[!ht]
    \caption{SD-GSEMO with safety parameter $R$}
    {\small
    \begin{algorithmic}[1]
    \STATE {Generate $x\in\{0,1\}^n$ uniformly at random, and $P\leftarrow \{x\}$}
    \STATE {$r\leftarrow 1$ and $u \leftarrow 0$}
    \LOOP
    \STATE {Uniformly and randomly select one individual $x$ from $P$}
    \STATE {Generate $x'$ via independently flipping each bit value of $x$ with probability $r/n$}
    \STATE {$u \leftarrow u+1$}
    \IF {there is no $y \in P$ such that $x' \preceq y$}
    \STATE {$P=\{z\in P \mid z \npreceq x'\} \cup \{x'\}$}
    \STATE {$r\leftarrow 1$ and $u \leftarrow 0$}
    \ENDIF
    \IF {$u > 2|P|(\frac{en}{r})^{r} \ln (nR)$}
    \STATE {$r\leftarrow \min\{r+1,\frac n2\}$ and $u \leftarrow 0$}
    \ENDIF
    \ENDLOOP
    \end{algorithmic}
    \label{alg:sdgsemo}
    }
\end{algorithm}

\subsection{Runtime Analysis of the SD-GSEMO on $\ojzj_{n,k}$}

We now analyze the runtime of the SD-GSEMO on the \ojzj function class. This will show that its expected runtime is by at least a factor of $\Omega(k)$ smaller than the one of the heavy-tailed GSEMO (which was a factor of $k^{\Omega(k)}$ smaller than the one of the standard GSEMO).

We extract one frequently used calculation into Lemma~\ref{lem:cal} to make our main proof clearer. The proof of Lemma~\ref{lem:cal} uses the following lemma. 
\begin{lemma}[Lemma~2.1~\cite{RajabiW20}]
Let $m,n \in \N$ and $m < n$. Then $\sum_{i=1}^m (\frac{en}{i})^i < \frac {n}{n-m} (\frac{en}{m})^m$.
\label{lem:lem21}
\end{lemma}

\begin{lemma}
Let $c\ge 0$, $a,n \in \N$, and $a \le \frac n2$. Then
\begin{align*}
\sum_{r=a+1}^{n/2}\frac{1}{n^{c(r-a)}} \sum_{i=1}^r \left(\frac {en}{i}\right)^i \le 2\left(\frac{en}{a}\right)^{a} \left(\frac{1}{n^{c-1}a}+ \frac{1}{2n^{2c-3}a^2}\right).
\end{align*}
\label{lem:cal}
\end{lemma}
\begin{proof}
We compute
\begin{align*}
\sum_{r=a+1}^{n/2}&\frac{1}{n^{c(r-a)}} \sum_{i=1}^r \left(\frac {en}{i}\right)^i \le \sum_{r=a+1}^{n/2}\frac{1}{n^{c(r-a)}} \frac {n}{n-r} \left(\frac{en}{r}\right)^r
\le 2\sum_{r=a+1}^{n/2}\left(\frac{e}{r}\right)^r\frac{n^r}{n^{c(r-a)}} \\
& \le {} 2 \left(\frac{e}{a+1}\right)^{a+1} \frac{n^{a+1}}{n^c} + 2\left(\frac n2 -a\right)\left(\frac{e}{a+2}\right)^{a+2}\frac{n^{a+2}}{n^{2c}}\\
& \le {} 2\left(\frac{en}{a}\right)^{a} \left( \frac{e}{n^{c-1}a} \left(1-\frac{1}{a+1}\right)^{a+1} + \frac{e^2}{2n^{2c-3}a^2} \left(1-\frac{2}{a+2}\right)^{a+2} \right)\\
& \le {} 2\left(\frac{en}{a}\right)^{a} \left(\frac{1}{n^{c-1}a}+ \frac{1}{2n^{2c-3}a^2}\right),
\end{align*}
where the first inequality uses Lemma~\ref{lem:lem21}.
\end{proof}

As for all algorithms discussed in this work, the most costly part of the optimization process is finding the two extremal points of the Pareto front. Borrowing many ideas from the proof of~\cite[Theorem~3.2]{RajabiW20}, we show the following estimate for the time to find these extremal points when the inner part of the Pareto front is already covered.

\begin{lemma}
Consider the SD-GSEMO optimizing the $\ojzj_{n,k}$ function. Consider the first time $T_0$ that all $(a,2k+n-a)$, $a\in[2k..n]$, are  {covered by the current population}. Then the two possibly missing elements $(a,2k+n-a)$, $a\in\{k,n+k\}$, of the Pareto front will be found in an expected number of at most $(n-2k+3)(\frac{en}{k})^{k}(\frac 32+ ( \frac{4k}{n} +\frac{12}{nk}) \ln (nR)) + 2k$ additional generations.
\label{lem:jumpsd}
\end{lemma}

\begin{proof}
Let $T'$ denote the first time (additional number of iterations) that one of $0^n$ and $1^n$ is found. Let phase $r$ be the period when the SD-GSEMO uses the bit-wise mutation rate $\frac rn$, and let $A_r$ denote the event that either $0^n$ or $1^n$ is found in phase $r$. With the pessimistic assumption that neither $0^n$ nor $1^n$ is found during phase $r$ for all $r < k$, we have
\begin{align*}
E[T']=\sum_{r=1}^{n/2} E[T'\mid A_r] \Pr[A_r]\le\sum_{r=k}^{n/2} E[T'\mid A_r] \Pr[A_r \mid r \ge k].
\end{align*}
Conditional on the event $A_r$, the time $T'$ to find the first extremal point includes all time spent in phases $1$ to $r-1$, that is, at most $\tilde{T}_1 + \dots + \tilde{T}_{r-1}$ with $\tilde{T}_i$ defined as in~\eqref{eq:trprime} for $|P|=n-2k+3$. In phase $r$, each iteration has a probability of at least $p_r = \frac{2}{n-2k+3}\left(\frac{r}{n}\right)^k\left(1-\frac rn\right)^{n-k}$ to find an extremal point (via choosing one of the boundary points of the inner region  {as} parent and then flipping exactly the right $k$ bits in it). Consequently, the time spent in phase $r$ follows a geometric law with success probability $p$ conditional on being at most $\tilde{T}_r$. Since for any random variable $X$ and any number $T$ such that $\Pr[X \le T] > 0$ we have $E[X \mid X \le T] \le E[X]$, the expected number of iterations spend in phase $r$ conditional on $A_r$ is at most $\frac {1}{p_r}$. 

From this, and recalling from Corollary~\ref{cor:population} that the population size is at most $n-2k+3$, we compute 
\begin{align}
E[T'\mid & A_k] \le \tilde{T}_1 + \dots + \tilde{T}_{k-1}  + \tfrac {1}{p_k} \nonumber\\
&\le \sum_{r=1}^{k-1} \left\lceil 2(n-2k+3) \left(\frac{en}{r}\right)^{r} \ln (nR) \right\rceil+\left(\frac{2}{n-2k+3}\left(\frac{k}{n}\right)^k\left(1-\frac kn\right)^{n-k}\right)^{-1}\nonumber\\
&\le {} 2(n-2k+3)\ln (nR)\frac{n}{n-k+1}\left(\frac{en}{k-1}\right)^{k-1} + k-1
+\frac{n-2k+3}{2}\left(\frac{en}{k}\right)^{k}\nonumber\\
&\le {} \frac{n-2k+3}{2}\left(\frac{en}{k}\right)^{k} \left( \frac{4k\ln (nR)}{en} \left(1+\frac{1}{k-1}\right)^{k-1} + 1\right)+k-1\nonumber\\
&\le {} \frac{n-2k+3}{2}\left(\frac{en}{k}\right)^{k} \left( \frac{4k\ln (nR)}{n}+ 1\right) +k -1,
\label{eq:halfT}
\end{align}
where the second inequality uses Lemma~\ref{lem:lem21}.

We observe that for $r\ge k$, the probability that during phase $r$ neither $0^n$ nor $1^n$ is found, is at most
\begin{align}
\bigg(1-\frac{2}{|P|}&\left(\frac{r}{n}\right)^k \left(1-\frac{r}{n}\right)^{n-k}\bigg)^{\lceil 2|P|(en)^r\ln(nR)/r^r \rceil}
\le \frac{1}{(nR)^4}.
\label{eq:fail}
\end{align}
Hence we have $\Pr[A_r] \le \prod_{i=k}^{r-1} \frac{1}{(nR)^4} \le (nR)^{-4(r-k)} \le n^{-4(r-k)}$ for $r>k$. Noting that the estimate $E[T'\mid A_r] \le \sum_{i=1}^{r} \tilde{T}_{i}$ holds trivially for $r<\frac n2$, and that it also holds for $r=\frac n2$ due to $\frac{1}{p_r} \le \lceil 2(n-2k+3) (2e)^{n/2} \ln (nR) \rceil= \tilde{T}_{\frac n2}$, we compute
\begin{align}
\sum_{r=k+1}^{n/2} & E[T'\mid A_r] \Pr[A_r] 
\le \sum_{r=k+1}^{n/2} \frac {1}{n^{4(r-k)}} \left(\sum_{i=1}^{r} \left\lceil 2(n-2k+3) \left(\frac{en}{i}\right)^{i} \ln (nR) \right \rceil \right)\nonumber\\
&\le {} 4(n-2k+3) \ln (nR) \left(\frac{en}{k}\right)^{k} \left(\frac{1}{n^3k}+\frac{1}{2n^5k^2}\right) + \sum_{r=k+1}^{n/2} \frac {r}{n^{4(r-k)}} \nonumber\\
&\le {} 4(n-2k+3) \ln (nR) \left(\frac{en}{k}\right)^{k} \left(\frac{1}{n^3k}+\frac{1}{2n^5k^2}\right) +\sum_{r=k+1}^{n/2} \frac {n}{2n^{4}} \nonumber\\
&\le {} 4(n-2k+3) \ln (nR) \left(\frac{en}{k}\right)^{k} \left(\frac{1}{n^3k}+\frac{1}{2n^5k^2}\right) +1,
\label{eq:anhalfT}
\end{align}
where we use Lemma~\ref{lem:cal} for the second inequality.

Noting $\Pr[A_k] \le 1$, together with (\ref{eq:halfT}) and (\ref{eq:anhalfT}), we have 
\begin{align*}
E[T'] \le \frac{n-2k+3}{2}\left(\frac{en}{k}\right)^{k} \left( \frac{4k\ln (nR)}{n}+ 1 + \left(\frac{1}{n^3k}+\frac{1}{2n^5k^2}\right) 8\ln (nR)\right) + k.
\end{align*}

We omit a precise statement of the very similar calculation for the time to generate the second extremal point. The only differences are that the expression $\frac{2}{n-2k+3}$ in the definition of $p$ has to be replaced by $\frac{1}{n-2k+3}$ and that the expression $\frac{2}{|P|}$ in (\ref{eq:fail}) has to be replaced by $\frac{1}{|P|}$, both due to the fact that now only one parent individual can generate the desired solution via a $k$-bit flip. With this, we obtain that the expected additional number of iterations to generate the second extremal point is at most
\begin{align*}
(n-2k+3)\left(\frac{en}{k}\right)^{k}\left(\frac{2k\ln (nR)}{n} +1+ \left(\frac{1}{nk}+\frac{1}{2nk^2}\right) 4\ln (nR)\right) + k.
\end{align*}
With $4(\frac{1}{nk}+\frac{1}{2nk^2})+8(\frac{1}{2n^3k}+\frac{1}{4n^5k^2}) \le \frac{12}{nk}$, this lemma is proven.
\end{proof}

Now we are ready to show the runtime for the SD-GSEMO on the $\ojzj_{n,k}$ function.

\begin{theorem}
The expected runtime of the SD-GSEMO optimizing the $\ojzj_{n,k}$ function is at most 
\[
\textstyle{(n-2k+3)(\frac{en}{k})^{k}(\frac 32+ ( \frac{4k}{n} +\frac{12}{nk}) \ln (nR))+3e(n-2k+3)(n\ln n +2(n-2)\ln(nR))}.
\]
For $k = o(n / \ln(nR))$, this bound is at most 
\[
\textstyle{(n-2k+3)\big((\frac 32 + o(1))(\frac{en}{k})^{k} + 3en (\ln n + 2\ln(nR))\big)}.
\]
\label{thm:sdgsemoojzj}
\end{theorem}

\begin{proof}
We first consider the expected time until we have all $(a,2k+n-a)$, $a\in[2k..n]$, in the current Pareto front. Our main argument is the following adaptation of the analysis of this period for the standard GSEMO in the proof of Theorem~\ref{thm:gsemoojzj}. The main argument there was that the time spent in this period can be estimated by a sum of at most $n-2$ waiting times for the generation of a particular Hamming neighbor of a particular member of the population. The same argument, naturally, is valid for the SD-GSEMO, and fortunately, we can also reuse the computation of these waiting times. Consider the time such a subperiod now takes with the SD-GSEMO. If we condition on the subperiod ending before the rate is increased above the initial value of $1/n$, then the expected time is at most the expected time of this subperiod for the standard GSEMO (this is the same argument as used in the proof of Lemma~\ref{lem:jumpsd}). Consequently, the total time of this period for the SD-GSEMO is at most the corresponding time for the standard GSEMO ($en(n-2k+3)(\ln n +1)+ 2en(n-2k+3) (\ln (n-k)+1)$ as determined in the proof of Theorem~\ref{thm:gsemoojzj}) plus the additional time caused by subperiods using rates $r/n$ with $r \ge 2$. 

We prove an upper bound for this time valid uniformly for all subperiods. Let phase $r$ be the time interval of this subperiod when the SD-GSEMO uses the bit-wise mutation rate $\frac rn$. We recall that at all times the population contains an individual $x$ such that there is a search point $y$ that can dominate at least one individual in the current population and that $y$ is a Hamming neighbor of $x$, that is, can be generated from $x$ by flipping a single bit.
Hence for all $r \ge 1$, the probability that this $y$ is not found in phase $r$ is at most
\begin{align*}
\left(1-\frac{1}{|P|}\left(1-\frac rn\right)^{n-1}\frac{r}{n}\right)^{\lceil 2|P|(en)^r\ln(nR)/r^r \rceil} \le \frac{1}{(nR)^2}.
\end{align*}
Thus with the probability at least $1-\frac{1}{(nR)^2}$, the desired search point $y$ will be found in phase $r$. Hence, if $y$ is not found in phase $r = 1$, then analogous to~\eqref{eq:anhalfT} in the proof of Lemma~\ref{lem:jumpsd}, also noting that the estimate $E[T'\mid A_r] \le \sum_{i=1}^{r} \tilde{T}_{i}$ holds trivially for $r<\frac n2$ and it also holds for $r=\frac n2$ since $|P|2^n \le \lceil 2|P| (2e)^{n/2} \ln (nR) \rceil= \tilde{T}_{\frac n2}$, the expected time spent in phases from $2$ to $n/2$ is at most
\begin{align*}
\sum_{r=2}^{n/2} \frac{1}{(nR)^{2(r-1)}}\sum_{i=1}^{r} \left( 2|P| \left(\frac{en}{i}\right)^{i} \ln (nR) \right) 
\le 4en|P| \ln(nR)\frac{3}{2n} = 6e|P|\ln(nR),
\end{align*}
where we use Lemma~\ref{lem:cal} for the first inequality. 

With this the total additional runtime caused by the stagnation-detection strategy compared with the GSEMO to find all $(a,2k+n-a)$, $a\in[2k..n]$ in the Pareto front is at most $(n-2)6e|P|\ln(nR)$. Therefore together with the runtime for GSEMO to find all $(a,2k+n-a)$, $a\in[2k..n]$ in the Pareto front in the proof of Theorem~\ref{thm:gsemoojzj}, we know the expected time when all $(a,2k+n-a)$, $a\in[2k..n]$ are contained in the current Pareto front is at most
\begin{align*}
en(n-2k+3)(\ln n +1)&+ 2en(n-2k+3) (\ln (n-k)+1)\\
+&{} (n-2)6e(n-2k+3)\ln(nR))\\
= {}&  {3e}(n-2k+3)(n\ln n +2(n-2)\ln(nR)).
\end{align*}

Together with the time to find the remaining two elements in the Pareto front in Lemma~\ref{lem:jumpsd}, this theorem is proven.
\end{proof}

Assume that, as suggested in~\cite{RajabiW20}, the control parameter $R$ is set to $n$. Then the dominating element of the upper bound in Theorem~\ref{thm:sdgsemoojzj} becomes $(n-2k+3)(\frac{en}{k})^k(\frac32 + 8(\frac kn+\frac{3}{nk})\ln n)$. Hence if $k=O(\frac{n}{\ln n})$,  then the runtime of SD-GSEMO on $\ojzj_{n,k}$ is $O((n-2k)(\frac{en}{k})^k)$.

\section{Experiments}
\label{sec:exp}

To understand the performance of the algorithms discussed in this work for concrete problem sizes (for which an asymptotic mathematical analysis cannot give definite answers), we now conduct a simple experimental analysis. Since the SEMO cannot find the Pareto front, we did not include it in this investigation. We did include the variant of the SD-GSEMO, denoted by SD-GSEMO-Ind, in which each individual has its own failure counter (see the discussion in Section~\ref{sec:sdgsemo}). Our experimental settings are the same for all algorithms.
\begin{itemize}
\item $\ojzj_{n,k}$: jump size $k=4$ and problem size $n=10,14,\dots,50$.
\item Maximal iterations for the termination: $n^{k+3}$. This number of iterations was reached in none of the experiments, i.e., the results reported are the true runtimes until the full Pareto front was found.
\item $\beta=1.5$ as suggested in~\cite{DoerrLMN17} for the power-law distribution in GSEMO-HTM.
\item $R=n$ for SD-GSEMO and SD-GSEMO-Ind as suggested in~\cite{RajabiW20}. 
\item $20$ independent runs for each setting. 
\end{itemize}

Figure~\ref{fig:ojzjrun} shows the median number of function evaluations of GSEMO, GSEMO-HTM, SD-GSEMO, and SD-GSEMO-Ind on the $\ojzj_{n,k}$ function. To see how the experimental results compare with our bounds, we also plot (i)~the curve $1.5e(n-2k)n^k$ corresponding to the bounds for the GSEMO in Theorems~\ref{thm:gsemoojzj} and~\ref{thm:lowerbound}, (ii)~the curve $(n-2k)(en)^k/k^{k-1}$ for the GSEMO-HTM with $\beta=1.5$; since the leading constant in Theorem~\ref{thm:gsemohojzj} is implicit, we chose a constant such that the curve matches the experimental data. We observe that for $n\ge 18$,  {Theorem}~\ref{cor:gsemohojzj} have given an upper bound of $\tfrac{ {4}e^{8\sqrt 2 +14}\beta}{(\beta-1)\sqrt \pi} \ge  {6\times10^{11}}$ for the leading constant, which is far larger than $1$ and indicates that further improvements of the leading constant estimate are possible, and (iii)~the curve $1.5(n-2k)(en)^k/k^{k}$ corresponding to the upper bound of SD-GSEMO with $R=n$ in Theorem~\ref{thm:sdgsemoojzj}. 

We clearly see that these curves, in terms of shape and, where known, in terms of leading constants, match well the estimates of our theoretical runtime results. We also see, as predicted by informal considerations, the similarity of the performance of the SD-GSEMO and the SD-GSEMO-Ind. Finally, our experiments show that the different runtime behaviors are already visible for moderate (and thus realistic) problem sizes and not only in the asymptotic sense in which they were proven. In particular, we observe a performance improvement by a factor of (roughly) $5$ through the use heavy-tailed mutation and by a factor of (roughly) $10$ with the stagnation-detection strategy.
\begin{figure}[!ht]
\centering
\includegraphics[width=4.0in]{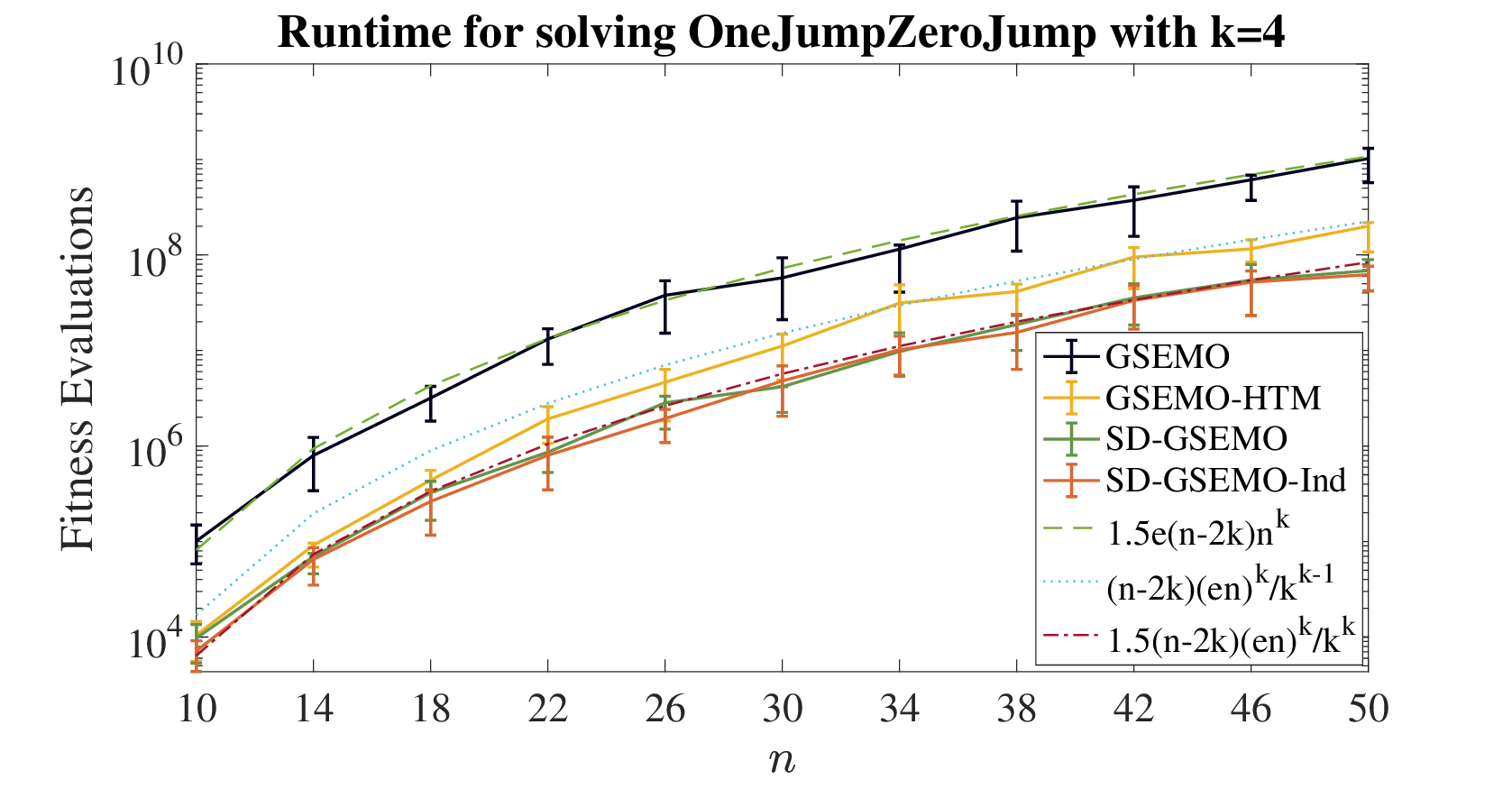}
\caption{The median number of function evaluations (with the first and third quartiles) of GSEMO, GSEMO-HTM, SD-GSEMO, and SD-GSEMO-Ind on $\ojzj_{n,k}$ with $k=4$ and $n=10:4:50$ in 20 independent runs.}
\label{fig:ojzjrun}
\end{figure}

\section{Conclusion and Outlook}
\label{sec:con}

Noting that, different from single-objective evolutionary computation, no broadly accepted multimodal benchmarks exist in the theory of MOEAs, we proposed the \ojzj benchmark, which is a multiobjective analogue of the single-objective jump problem. 

We use this benchmark to show that several insights from single-objective EA theory extend to the multiobjective world. Naturally, algorithms using 1-bit mutation such as the SEMO cannot optimize our benchmark, that is, cannot compute a population that covers the full Pareto front. In contrast, for all problem sizes $n$ and jump sizes $k\in [4..\frac n2 -1]$, the GSEMO covers the Pareto front in $\Theta((n-2k)n^k)$ iterations in expectation. Heavy-tailed mutation rates and the stagnation-detection mechanism of \cite{RajabiW22} give a runtime improvement over the standard GSEMO by a factor of $k^{\Omega(k)}$, with the stagnation-detection GSEMO slightly ahead by a small polynomial factor  {in~$k$}. These are the same gains as observed for the single-objective \jump benchmark with gap size~$k$. Our experiments confirm this performance ranking already for moderate problem sizes, with the stagnation-detection GSEMO slightly more ahead than what the asymptotically small advantage suggests. On the downside, adapting the stagnation-detection mechanism to MOEAs needs taking several design choices, among which the question how to treat indifferent solutions could be difficult for problems having larger plateaus of constant fitness.

Overall, this work suggests that the recently developed ideas to cope with multimodality in single-objective evolutionary optimization can be effective in multiobjective optimization as well. In this first work in this direction, we only concentrated on mutation-based algorithms. The theory of evolutionary computation has also observed that crossover and estimation-of-distribution algorithms can be helpful in multimodal optimization. Investigating to what degree these results extend into multiobjective optimization is clearly an interesting direction for future research. 

Also, we only covered very simple MOEAs in this work. Analyzing more complex MOEAs amenable to mathematical runtime anlyses, such as the decomposition-based MOEA/D~\cite{ZhangL07,LiZZZ16} or the \NSGA~\cite{DebPAM02,ZhengLD22} would be highly interesting. For the MOEA/D, this would most likely require an adaptation of our benchmark problem. Since the difficult-to-find extremal points of the front are just the solutions of the single-objective sub-problems, and thus the two problems that naturally are part of the set of subproblems regarded by the MOEA/D, this algorithm might have an unfair advantage on the \ojzj problem. 


\emph{Work conducted after this research:} After the publication of~\cite{DoerrZ21aaai}, the following works have used the \ojzj benchmark. In~\cite{DoerrQ22ppsn}, the performance of the \NSGA, both with standard-bit mutation and with heavy-tailed mutation, on the \ojzj benchmark was analyzed. When the  {population} size $N$ is at least four times the size $n-2k+3$ of the Pareto front and $2 < k \le n/4$, then the standard \NSGA finds the Pareto front in time $O(N n^k)$. Hence the \NSGA and the GSEMO satisfy the same asymptotic runtime guarantees when $N$ is chosen asymptotically optimal (that is, $N = \Theta(n-2k)$). With heavy-tailed mutation, again a $k^{\Omega(k)}$ improvement is obtained. In this work, an \NSGA was regarded that does not use crossover, but it is clear that the same upper bounds could have been shown when crossover with rate less than one was used. Doerr and Qu\cite{DoerrQ23LB} show that the upper bound of $O(N n^k)$ for the \NSGA with standard-bit mutation is asymptotically tight in many cases: When $N = o(n^2/k^2)$, the runtime is $\Omega(N n^k)$. Consequently, in this regime the \NSGA does not profit from larger population sizes (not even when regarding the parallel runtime, that is, the number of iterations). Doerr and Qu~\cite{DoerrQ23crossover} prove that uniform crossover can give super-constant speed-ups for the \NSGA optimizing \ojzj functions. This result is not totally surprising given that Dang et al.~\cite{DangFKKLOSS18} have shown that crossover speeds up the \mpoea on single-objective jump functions, but the reason for the speed-ups shown in~\cite{DoerrQ23crossover} is different (and, in fact, can be translated to different, and sometimes stronger, speed-ups for the \mpoea on jump functions). These recent results show that a jump-like benchmark, as proposed in this work, is useful in multiobjective evolutionary computation. 

\section*{Acknowledgments}
This work was supported by a public grant as part of the Investissement d'avenir project, reference ANR-11-LABX-0056-LMH, LabEx LMH.

This work was also supported by Science, Technology and Innovation Commission of Shenzhen Municipality (Grant No. GXWD20220818191018001), Guangdong Basic and Applied Basic Research Foundation (Grant No. 2019A1515110177), Guangdong Provincial Key Laboratory (Grant No. 2020B121201001), the Program for Guangdong Introducing Innovative and Enterpreneurial Teams (Grant No. 2017ZT07X386), Shenzhen Science and Technology Program (Grant No. KQTD2016112514355531).


\newcommand{\etalchar}[1]{$^{#1}$}

}
\end{document}